\documentclass[journal]{IEEEtran}
\usepackage{amsmath,amsfonts}
\usepackage{array}
\usepackage[caption=false,font=normalsize,labelfont=sf,textfont=sf]{subfig}
\usepackage{textcomp}
\usepackage{stfloats}
\usepackage{url}
\usepackage{verbatim}
\usepackage{graphicx}

\usepackage[ruled, vlined, commentsnumbered, linesnumbered]{algorithm2e}
\usepackage{algorithmic}
\usepackage{amsthm}

\usepackage[most]{tcolorbox}

\newcommand\myeq{\mathrel{\overset{\makebox[0pt]{\mbox{\normalfont\tiny\sffamily d}}}{=}}}

\usepackage{color}
\usepackage{xcolor}

\usepackage{booktabs,tabularx}
\usepackage{multirow}
\usepackage{array}

\usepackage{hyperref}

\usepackage[numbers,sort&compress]{natbib}

\newtheorem{definition}{Definition}
\newtheorem{assumption}{Assumption}
\newtheorem{theorem}{Theorem}

\newtheorem{proposition}{Proposition}

\newcommand{\eg}{{\em e.g.}}           
\newcommand{\ie}{{\em i.e.}}           
\newcommand{\etc}{{\em etc}}         

\newcommand{\re}{\color{black}}

\begin{document}

\title{Permutation-Equivariant 2D State Space Models: Theory and Canonical Architecture for Multivariate Time Series}

\author{Seungwoo Jeong
        and~Heung-Il~Suk,~\IEEEmembership{Senior Member,~IEEE}


\thanks{S. Jeong is with the Department of Artificial Intelligence, Korea University, Seoul 02841, Republic of Korea (e-mail: sw\_jeong@korea.ac.kr).}
\thanks{H.-I. Suk is with the Department of Artificial Intelligence, Korea University, Seoul 02841, Republic of Korea and the corresponding author (e-mail: hisuk@korea.ac.kr).}
}

\markboth{Journal of \LaTeX\ Class Files,~Vol.~14, No.~8, March~2026}%
{Jeong and Suk: Permutation-Equivariant 2D State Space Models: Theory and Canonical Architecture for Multivariate Time Series}


\maketitle

\begin{abstract}
Multivariate time series (MTS) modeling often implicitly imposes an artificial ordering over variables, violating the inherent exchangeability found in many real-world systems where no canonical variable axis exists. We formalize this limitation as a violation of the permutation symmetry principle and require state-space dynamics to be permutation-equivariant along the variable axis. In this work, we theoretically characterize the complete canonical form of linear variable coupling under this symmetry constraint. We prove that any permutation-equivariant linear 2D state-space system naturally decomposes into local self-dynamics and a global pooled interaction, rendering ordered recurrence not only unnecessary but structurally suboptimal. Motivated by this theoretical foundation, we introduce the Variable-Invariant Two-Dimensional State Space Model (VI 2D SSM), which realizes the canonical equivariant form via permutation-invariant aggregation. This formulation eliminates sequential dependency chains along the variable axis, reducing the dependency depth from $\mathcal{O}(C)$ to $\mathcal{O}(1)$ and simplifying stability analysis to two scalar modes. Furthermore, we propose VI 2D Mamba, a unified architecture integrating multi-scale temporal dynamics and spectral representations. Extensive experiments on forecasting, classification, and anomaly detection benchmarks demonstrate that our model achieves state-of-the-art performance with superior structural scalability, validating the theoretical necessity of symmetry-preserving 2D modeling.
\end{abstract}


\section{Introduction}
\IEEEPARstart{T}{ime} series data are ubiquitous across scientific and industrial domains, ranging from climate science~\citep{mudelsee2010climate}, finance~\citep{shi2024mambastock}, and biomedical signals~\citep{jeong2024deep}. Those data streams are characterized by complex temporal dynamics, seasonal regularities, and, crucially, inter-variable dependencies among multiple variables. In the multivariate setting, effective modeling requires capturing three distinct structural properties: (\romannumeral 1) the coexistence of long-term trends and short-term fluctuations, (\romannumeral 2) dynamic and evolving correlations among variables, and (\romannumeral 3) distinct spectral manifestations of these patterns in the frequency domain~\citep{qiu2025comprehensive}. Developing model architectures that simultaneously address these multi-scale temporal, inter-variable, and spectral patterns-while remaining computationally tractable over long horizons-remains a fundamental challenge in forecasting~\citep{zhou2021informer}, classification~\citep{jeong2023deep}, and anomaly detection~\citep{luo2024moderntcn}.

To address these challenges, deep learning architectures have evolved significantly. While Convolutional Neural Networks (CNNs) and Recurrent Neural Networks (RNNs) effectively capture local and autoregressive patterns, they often struggle to model long-range dependencies~\citep{shen2020timeseries, franceschi2019unsupervised}. Transformers have mitigated this issue via global self-attention~\citep{wen2023transformers}, yet their quadratic complexity with respect to sequence length limits scalability~\citep{wang2025mamba}. Recently, State Space Models (SSMs) have emerged as a compelling alternative, offering structured recurrences with linear time complexity~\citep{gu2021efficiently}. Notably, Mamba ~\citep{gu2023mamba} has demonstrated state-of-the-art performance by introducing input-selective state transitions. However, conventional SSMs are inherently one-dimensional; they evolve solely along the temporal axis, lacking an explicit mechanism to model the dynamic interactions between variables.

This limitation has motivated the development of Two-Dimensional SSMs (2D SSMs)~\citep{zhang20252dmamba, baron20242}, which extend recurrence to both temporal and variable axes. Despite their improved expressivity, existing 2D SSMs typically rely on sequential scanning along the variable axis (Fig.~\ref{fig:intro}). This approach implicitly treats variable indices as ordered coordinates, imposing a Markov chain structure where none may exist. While such an ordering is natural in spatial domains like image processing, multivariate time series generally satisfy exchangeability: variable indices act as identifiers rather than geometric coordinates. Consequently, sequential scanning introduces an artificial spatial inductive bias, rendering the model sensitive to variable permutation and inducing sequential dependency chains that inhibit parallel computation.

\begin{figure*}[t]
    \centering
    \includegraphics[width=1\linewidth]{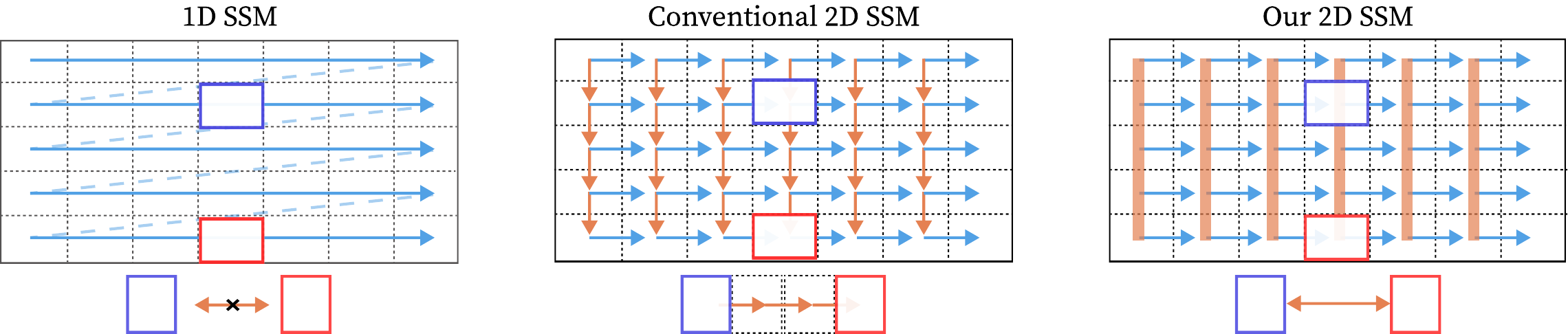}
    \caption{\textbf{Left:} 1D SSM models only along the temporal axis, overlooking dependencies across variables. \textbf{Middle:} Conventional 2D SSM captures inter-variable correlations through sequential scans, but still imposes artificial ordering and struggles with distant relationships. \textbf{Right:} Our method performs global aggregation over variables, enabling simultaneous and order-free modeling of inter-variable relationships over time.}
    \label{fig:intro}
\end{figure*}

In this work, we address this fundamental mismatch by formalizing the permutation symmetry principle for multivariate time series. We posit that a well-specified multivariate dynamical model must be permutation-equivariant along the variable axis. Under this symmetry constraint, we theoretically characterize the complete class of admissible linear inter-variable couplings. Specifically, we prove that any permutation-equivariant linear 2D state-space system necessarily decomposes into local self-dynamics and a global pooled interaction term. This characterization reveals that the ordered recurrence employed by prior 2D SSMs is not merely unnecessary but structurally incompatible with the inherent symmetry of the data.
Guided by this theoretical insight, we introduce the Variable-Invariant Two-Dimensional State Space Model (VI 2D SSM). Instead of performing sequential updates across variables, our formulation employs a permutation-invariant aggregation mechanism to construct a global descriptor that mediates inter-variable interactions. We demonstrate that this architecture is a direct realization of the canonical equivariant form derived from {\re our theoretical characterization}.

Beyond theoretical consistency, permutation symmetry yields significant computational advantages. By eliminating ordered recurrence, we reduce the dependency depth along the variable axis from $\mathcal{O}(C)$ to $\mathcal{O}(1)$, enabling fully parallelized computation across variables. Furthermore, the symmetry-constrained coupling simplifies stability analysis, allowing us to decompose the dynamics into mean and difference modes governed by just two scalar constraints. Building on this foundation, we propose VI 2D Mamba, a unified architecture that integrates the proposed VI 2D SSM with multi-scale temporal pathways (long- and short-term) and a frequency-domain branch via adaptive gating. This design ensures comprehensive coverage of heterogeneous temporal and spectral structures while strictly preserving permutation invariance.

Extensive experiments on forecasting, classification, and anomaly detection benchmarks demonstrate that the proposed method achieves competitive or superior performance compared to state-of-the-art baselines. These results validate the theoretical necessity of symmetry-preserving modeling and highlight its practical benefits in terms of scalability and robustness.

Our contributions are summarized as follows:

\begin{itemize}
\item \textbf{Permutation symmetry formalization for multivariate 2D SSMs.}
We formalize multivariate time-series modeling under the principle of variable-axis exchangeability. We establish permutation equivariance as a fundamental constraint for valid 2D state-space dynamics in non-spatial domains.

\item \textbf{Canonical characterization of equivariant coupling.}
We provide a theoretical proof that any permutation-equivariant linear inter-variable state coupling admits a unique canonical decomposition into self-dynamics and global pooled interaction. This result fully characterizes the admissible class of symmetry-compatible linear 2D dynamics.

\item \textbf{Symmetry-preserving 2D state-space realization.}
We introduce VI 2D SSM, which implements the canonical equivariant form via permutation-invariant aggregation. 
We show that this formulation eliminates artificial dependency chains and simplifies stability analysis to verifiable scalar conditions.

\item \textbf{Structural scalability and empirical validation.}
We demonstrate that our symmetry-constrained approach reduces variable-axis dependency depth from $\mathcal{O}(C)$ to $\mathcal{O}(1)$, enabling full parallelization. Empirical evaluations across diverse benchmarks confirm our proposed model’s superior efficiency and performance.
\end{itemize}

\section{Related Works}
\label{related work}
\subsection{Neural Models for Time Series Modeling}

Research on deep learning for multivariate time series has evolved from modeling purely temporal dynamics to explicitly capturing cross-variable dependencies~\citep{qiu2025comprehensive}. 
Early CNN- and RNN-based approaches extracted local patterns and autoregressive structures but were limited in modeling long-range dependencies~\citep{lai2018modeling, shen2020timeseries, franceschi2019unsupervised, he2019temporal}. 
Transformer-based architectures advanced global context modeling via self-attention, yet their quadratic complexity with sequence length poses scalability challenges despite sparsification and kernelization efforts~\citep{keles2023computational}. 
To improve efficiency, convolutional alternatives~\citep{luo2024moderntcn} and SSMs~\citep{gu2021efficiently, gu2023mamba, patro2024simba} have been proposed, offering linear or near-linear complexity.

However, most existing formulations evolve solely along the temporal axis and do not explicitly model time-varying cross-variable interactions. 
To address this, recent extensions of Mamba—including TimePro, CMamba~\citep{zeng2024cmamba}, Simba~\citep{patro2024simba}, and GrootVT~\citep{xiao2024grootvl}—incorporate variable-level priors, often through heuristic coupling mechanisms.

This has further motivated 2D SSMs, inspired by classical formulations such as Roesser’s model~\citep{kung1977new}. 
Recent variants~\citep{baron20242, zhang20252dmamba, behrouz2024chimera} couple temporal and variable dynamics to better capture multivariate correlations. 
Among them, Chimera~\citep{behrouz2024chimera} improves cross-variable modeling but relies on sequential scanning along the variable axis, thereby introducing ordering dependence and limiting parallelism.

These limitations motivate the need for symmetry-consistent 2D state-space formulations, which we address through permutation-invariant variable coupling.

\subsection{Permutation-Invariant and Equivariant Modeling}

Permutation symmetry has been extensively studied in machine learning under the framework of invariant and equivariant neural networks.
Deep Sets~\citep{zaheer2017deep} established that any permutation-invariant function over a set can be decomposed into a sum-aggregation followed by a transformation, providing a canonical representation for set functions.
Subsequent works extended this principle to permutation-equivariant mappings~\citep{bloem2020probabilistic}, as well as attention-based architectures such as Set Transformer~\citep{lee2019set}, which model interactions over unordered elements. Permutation equivariance has also been studied more broadly in group-equivariant neural networks, where symmetry constraints are enforced through structured operators and representation theory~\citep{cohen2016group, kondor2018generalization}.

Unlike prior works that focus on static function approximation over sets, our work studies permutation symmetry in the context of dynamical state-space systems for multivariate time series.
Rather than constructing invariant feature mappings, we characterize the complete class of linear permutation-equivariant inter-variable dynamics within 2D SSMs.

\section{Preliminaries}
\label{preliminaries}
We consider a multivariate time series $X \in \mathbb{R}^{C \times T}$ consisting of $C$ variates and $T$ time steps. We index the temporal dimension by $t \in \{1, \dots, T\}$ and the variable dimension by $c \in \{1, \dots, C\}$, where the observation at position $(t, c)$ is denoted by $x(t, c)$.

Our goal is to learn a representation mapping $f_{\theta}$ that captures both temporal dynamics and inter-variable dependencies. Motivated by the success of structured SSMs in sequence modeling, we adopt a 2D state-space formulation. This architecture maintains two coupled latent states: a horizontal state $h_h(t, c) \in \mathbb{R}^{d_h}$ describing evolution along the temporal axis, and a vertical state $h_v(t, c) \in \mathbb{R}^{d_v}$ describing evolution along the variable axis.

\subsection{Continuous-Time 2D State Space Models}
The theoretical foundation of 2D SSMs lies in the Roesser model~\citep{kung1977new}, which generalizes standard linear systems to two dimensions. In the continuous domain, the joint evolution of the horizontal and vertical states is governed by the following coupled partial differential equations (PDEs):
\begin{align}
\label{eq1}
    \frac{\partial h_h(t, c)}{\partial t} &= (A_h h_h(t, c),\;A_{hv}h_v(t, c)) + B_h x(t, c),
    \\ 
    \label{eq2}
    \frac{\partial h_v(t, c)}{\partial c} &= (A_v h_v(t, c),\; A_{vh}h_h(t, c)) + B_v x(t, c),
    \\ 
    \label{eq3}
    y(t, c) &= C_h h_h(t, c) + C_v h_v(t, c).
\end{align}
Here, $A_{\bullet}$ and $B_{\bullet}$ are state transition and input projection matrices, respectively. The cross-coupling terms $A_{hv}$ and $A_{vh}$ allow information to flow between the temporal and variable axes, enabling the model to capture complex spatiotemporal correlations.

\subsection{Discretization and the Ordering Problem}
To apply this continuous framework to digital signals, the system must be discretized. Recent selective-scan models, such as Mamba~\citep{gu2023mamba} and Chimera~\citep{behrouz2024chimera}, employ the Zero-Order Hold (ZOH) method. For the temporal axis, given a sampling step $\Delta_t$, the continuous parameters $(A_h, B_h)$ are transformed into discrete parameters $(\bar{A}_h, \bar{B}_h)$ via $\bar{A}_h = \exp(\Delta_t A_h)$ and $\bar{B}_h = (\Delta_t A_h)^{-1}(\exp(\Delta_t A_h) - I) \Delta_t B_h$.

This discretization naturally yields a recurrence relation. Existing 2D SSMs extend this recurrence to the variable axis as well:
\begin{align}
\label{eq4}
 h_h[t+1, c] &= \bar{A}_h h_h[t, c] + \bar{B}_h x[t+1,c], \\ 
 \label{eq5}
 h_v[t, c+1] &= \bar{A}_v h_v[t, c] + \bar{B}_v x[t, c+1],
 \\
 \label{eq6}
 y_{v, t} &= C_hh_h + C_vh_v,
\end{align}
While this sequential scanning is intuitive for spatial data (\eg, images) where $c$ and $c+1$ are geometric neighbors, it imposes a fundamentally flawed inductive bias on multivariate time series.
\begin{itemize}
    \item Artificial Ordering: Unlike the temporal axis, which possesses a strict causal order ($t \to t+1$), the variable axis in most multivariate datasets is exchangeable. The index $c$ serves merely as an identifier, not a coordinate. Imposing an update rule $h_v[c] \to h_v[c+1]$ creates an artificial dependency chain, making the model sensitive to the permutation of input variables.
    \item Sequential Bottleneck: The recurrence along $c$ forces sequential computation, preventing parallelization across variables and limiting scalability to high-dimensional systems.
\end{itemize}
These limitations motivate our proposed formulation, which replaces the ordered vertical recurrence with a permutation-invariant global coupling mechanism.

\section{Symmetry-Constrained 2D State Space Model}
\label{proposed method}
In this section, we formalize the symmetry principle underlying variable-invariant modeling and characterize the canonical form of permutation-equivariant inter-variable coupling. Building upon this theoretical foundation, we then construct the Variable-Invariant Two-Dimensional State Space Model (VI 2D SSM) as a structural realization of the derived canonical form. Finally, we present the VI 2D Mamba architecture, which extends this symmetry-preserving mechanism into a multi-scale spectral–temporal deep learning framework.

\subsection{Symmetry Principle and Canonical Coupling}
\label{theoretical}
We first establish the symmetry principle governing permutation-equivariant modeling and derive the canonical form of admissible inter-variable coupling. We then examine the resulting computational structure and stability properties, laying the theoretical groundwork for the Variable-Invariant 2D SSM introduced in the subsequent section.

\vspace{5pt}\noindent\textbf{Exchangeable Dynamical Systems.}
We formalize the modeling assumption under which permutation symmetry is desirable. We then characterize the class of admissible inter-variable dynamics consistent with permutation equivariance.

\begin{definition}[Variable-Axis Exchangeability]
    We say that the data-generating distribution of the multivariate time series  $X$ is exchangeable along the variable axis if $X \myeq X^\pi$, $\forall \pi\in S_C$. That is, permuting the variable indices does not change the joint distribution.
\end{definition}
This reflects the absence of a canonical ordering among variables. In many multivariate domains, variable indices serve as identifiers rather than ordered spatial coordinates. Thus, modeling assumptions that impose sequential structure along the variable axis introduce artificial inductive bias.

We therefore adopt the following modeling principle:
\begin{assumption}[Symmetry Principle]
    Under exchangeability, a well-specified model should be permutation-equivariant along the variable axis.
\end{assumption}

Conventional 2D state-space models impose ordered recurrence across variables, thereby introducing index-sensitive structure.
In contrast, the proposed formulation satisfies permutation equivariance by construction.

\vspace{5pt}\noindent\textbf{Expressive Power over Exchangeable Linear 2D Systems.}
We now characterize the class of linear variable-coupled state updates that satisfy permutation equivariance along the variable axis. Our goal is to determine the canonical form of linear dynamics that are compatible with the symmetry principle.

We consider linear vertical-state updates for the form 
\begin{align}
    h_v(t+1, c) = \sum^C_{j=1}M_{c, j}h_v(t, j)+\sum^C_{j=1}N_{c,j}x(t, j),
\end{align}
where $\{M, N\}\in\mathbb{R}^{C\times C}$ governs inter-variable state/input coupling. This represents the general linear update across the variable axis at a fixed time step.

Applying permutation to both sides of the update yields the algebraic condition: 
\begin{align}
    MP_\pi = P_\pi M,\quad NP_\pi = P_\pi N,\;\;\; \forall \pi \in S_C.
\end{align}
Thus, both $M$ and $N$ must commute with every permutation matrix. 

We now characterize the structure of permutation-equivariant
linear state coupling along the variable axis. For clarity, we analyze the state coupling matrix $M$ and omit the input coupling term.
\begin{theorem}[Characterization of Permutation-Commuting Matrices]
    Let $M \in \mathbb{R}^{C\times C}$. Then, 
    \begin{align}
        MP_\pi = P_\pi M \quad \forall \pi \in S_C 
    \end{align}
    if and only if 
    \begin{align}
        M = \alpha I_C + \beta \mathbf{11}^\top
    \end{align}
    for some scalars $\alpha, \beta \in \mathbb{R}$ where $I_C$ is the identity matrix and $\mathbf{1}$ is the column vector of all ones. 
    
    \noindent In particular, any linear permutation-equivariant state coupling must take the form
    \begin{align}
        h_v(t+1, c) = \alpha h_v(t, c) + \beta \sum^C_{j=1} h_v(t, j) + N x (t, c).
    \end{align}
\end{theorem}
The proof is provided in the Appendix~\ref{appendix:thm1}. This result characterizes the canonical form of linear inter-variable coupling under permutation equivariance. In the next section, we construct a state-space realization that adheres to this structure.
\begin{tcolorbox}[colback=gray!5,colframe=black!50,boxrule=0.5pt]
Permutation symmetry uniquely constrains admissible linear inter-variable dynamics to the canonical form 
$\alpha I_C + \beta \mathbf{1}\mathbf{1}^\top$.
\end{tcolorbox}

\vspace{5pt}\noindent\textbf{Computational Structure.}
Beyond expressivity, permutation symmetry fundamentally alters the computational structure along the variable axis.

\begin{theorem}[Reduction of Variable-Axis Dependency Depth]
Consider a 2D state-space system with $C$ variables.
\begin{enumerate}
    \item If inter-variable coupling is implemented via ordered recurrence (\eg, $h_v(t,c+1)$ depends on $h_v(t,c)$), then the update at each fixed time $t$ induces a dependency chain of length $C$ along the variable axis.
    \item In contrast, any formulation that mediates inter-variable interaction through a permutation-invariant aggregated descriptor $\psi(t)$ admits $\mathcal{O}(1)$ variable-axis dependency depth, since no sequential dependency is introduced once $\psi(t)$ is computed.
\end{enumerate}
\end{theorem}
This structural reduction directly enables parallel updates across variables, a property that will be leveraged by the proposed VI 2D SSM and validated empirically in Section~\ref{analysis:eff}.

\vspace{5pt}\noindent\textbf{Stability of Discretized Dynamics.} We analyze stability of the discretized state updates under permutation-equivariant coupling.

\noindent Consider the continuous-time vertical state dynamics
\begin{align}
    \dot{h}_v(t,c) = A_v h_v(t,c),
\end{align}
where $A_v \in \mathbb{R}^{d \times d}$.

The discretized update with step size $\Delta > 0$ under ZOH is
\begin{align}
    h_v[t+1,c] = \bar A_v h_v[t,c], 
    \quad 
    \bar A_v = e^{\Delta A_v}.
\end{align}

\begin{theorem}[Discrete Stability under ZOH]
If the continuous-time matrix $A_v$ is Hurwitz, i.e.,
\[
\Re(\lambda_i(A_v)) < 0 \quad \forall i,
\]
then for any $\Delta > 0$, the discretized matrix
\[
\bar A_v = e^{\Delta A_v}
\]
satisfies
\[
\rho(\bar A_v) < 1,
\]
where $\rho(\cdot)$ denotes the spectral radius.
\end{theorem}

Under the canonical permutation-equivariant coupling
\begin{align}
M = \alpha I_C + \beta \mathbf{1}\mathbf{1}^\top,
\end{align}
the full state update matrix decomposes into two invariant subspaces:
\begin{enumerate}
    \item The zero-sum subspace (orthogonal to $\mathbf{1}$),
    \item The span of $\mathbf{1}$.
\end{enumerate}
The eigenvalues along these subspaces are:
\begin{align}
\lambda_{\text{diff}} = \alpha,
\quad
\lambda_{\text{mean}} = \alpha + C\beta.
\end{align}
Hence stability requires
\begin{align}
|\alpha| < 1,
\quad
|\alpha + C\beta| < 1.
\end{align}
Thus, permutation symmetry reduces the stability analysis to two scalar constraints corresponding to the mean and difference modes. Importantly, symmetry does not guarantee stability per se. Rather, it constrains the dynamics to evolve within a two-mode invariant decomposition, reducing structural complexity and simplifying both stability and optimization analysis.

Taken together, the preceding analyses establish that permutation symmetry uniquely constrains admissible linear inter-variable dynamics to the canonical form 
$\alpha I_C + \beta \mathbf{1}\mathbf{1}^\top$. 
This characterization reveals that equivariant coupling necessarily decomposes into local self-dynamics and a global interaction term.
In the next section, we construct a state-space realization that adheres to this canonical structure, yielding a symmetry-consistent and computationally efficient formulation of inter-variable dynamics.

\begin{figure*}
    \centering
    \includegraphics[width=\linewidth]{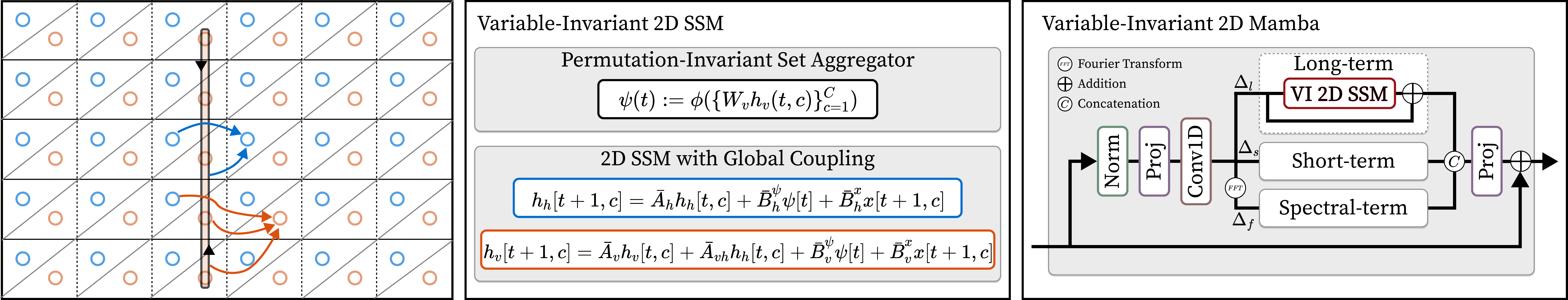}
    \caption{Overview of the proposed Variable-Invariant 2D Mamba. \textbf{Left:} Variable-Invariant 2D SSM with global coupling. A permutation-invariant set aggregator enforces invariance to variable ordering, after which the 2D state space model captures jointly coupled dynamics along temporal and variable dimensions. The globally coupled state transitions enable structured long-range dependency modeling while preserving symmetry across variables. \textbf{Right:} Variable-Invariant 2D Mamba architecture. The framework decomposes representation learning into complementary long-term, short-term, and spectral components. These components are fused to form a unified representation that simultaneously captures global temporal dynamics, local variations, and frequency-domain structure under a variable-invariant formulation.}
\end{figure*}

\subsection{Variable-Invariant 2D SSM}
\label{2dssm}
\subsubsection{Global Interaction Field Formulation}
Let $z(t, :) = [z(t, 1), \ldots, z(t, C)] \in \mathbb{R}^{C \times d_z}$ denote the ensemble of a variable-indexed representation at time $t$. To eliminate the artificial ordering imposed by sequential scanning ($c \to c+1$), we introduce a global interaction field $\psi(t)$, derived via a permutation-invariant aggregation:
\begin{align}
\label{eq7}
\psi(t) := \phi(\{W_v z(t, c)\}_{c=1}^C),
\end{align}
where $W_{v}$ is a learnable projection matrix and $\phi(\cdot)$ is a permutation-invariant set function (e.g., mean, sum, or attention-based pooling) such that $\phi(\{z_c\}) = \phi(\{z_{\pi(c)}\})$ for any permutation $\pi \in S_C$.

\begin{proposition}[Permutation invariance of the pooled summary]
\label{proposition:1}
Let $\pi$ be any permutation of $\{1,\ldots, C\}$ and define 
$\psi_\pi(t) := \phi\!\left(\{W_v z(t,\pi(c))\}_{c=1}^C\right)$.
Assume (i) $W_v$ is variable-shared (independent of $c$), and (ii) $\phi$ is permutation-invariant on multisets.
Then $\psi_\pi(t) = \psi(t)$ for all $t$.
\end{proposition}
This field $\psi(t)$ serves as a global descriptor that mediates information exchange across variables without imposing any sequential ordering. Once $\psi(t)$ is computed via a permutation-invariant aggregation, the subsequent state updates for different variables can be executed in parallel across the vertical axis. The full proof is provided in the Appendix~\ref{appendix:proof_p1}.

\subsubsection{Coupled Continuous-Time Dynamics}
We redefine the 2D state-space dynamics by coupling the temporal (horizontal) and variable (vertical) states through this global field. Unlike the Roesser model, which defines evolution along both axes ($t$ and $c$), our model evolves strictly along the temporal axis $t$, while the variable axis is governed by instantaneous global coupling:

\begin{align}
    \frac{\partial h_h(t, c)}{\partial t}
= &A_h h_h(t, c) + A_{h\psi}\psi(t)+ B_h x(t, c), \label{eq9}\\\nonumber
\frac{\partial h_v(t, c)}{\partial t}
= &A_v h_v(t, c) + A_{v\psi}\psi(t) \\
& + A_{vh} h_h(t, c) + B_v x(t,c).
   \label{eq10}
\end{align}

Here, $A_{h\psi}$ and $A_{v\psi}$ are coupling matrices that inject global context into local dynamics. Crucially, the derivative $\partial / \partial c$ is removed. Each variable $c$ evolves according to its own local history ($h_h, h_v$) and the global context $\psi(t)$, ensuring that the update rule is identical and exchangeable for all $c$.
\begin{proposition}
\label{proposition:2}
    Assume $A_h, A_v, A_{h\psi}, A_{v\psi}, A_{vh}, B_h, B_v$ are variable-shared (independent of $c$), and let $\psi(t)$ be defined as above with a permutation-invariant $\phi$ (Prop.~\ref{proposition:1}). For any permutation $\pi$ of $\{1, \ldots, C\}$, consider the permuted inputs $x^\pi(t, c) :=x(t, \pi(c))$ and permuted initial states $h^\pi(0, c):=h(0, \pi(c))$. Then the unique solutions to Eq. \ref{eq9}-\ref{eq10} satisfy $h^\pi_h(t, c) = h_h(t, \pi(c))$, $h^\pi_v(t, c)=h_v(t, \pi(c))$ $\forall\; t, c$,
    i.e., the dynamics are permutation equivariant with respect to the variable axis.
\end{proposition}
With variable-shared parameters, the permuted system is a mere relabeling of indices; hence, the solutions are correspondingly permuted. The full proof is provided in the Appendix~\ref{appendix:proof_p2}.

\subsubsection{Discrete Realization via Parallel Scanning}
Discretizing the system using ZOH method with sampling step $\Delta>0$, we obtain the following recurrence:
\begin{align}
    h_h[t, c]
&= \bar{A}_h h_h[t-1, c]
   + \bar{B}_h^{\psi} \psi[t]
   + \bar{B}_h^x x[t, c] \label{eq12}\\ \nonumber
h_v[t, c]
&= \bar{A}_v h_v[t-1, c]
   + \bar{B}_v^{\psi} \psi[t] \\
&\quad
   + \bar{A}_{vh} h_h[t-1, c]
   + \bar{B}_v^x x[t, c].
   \label{eq13}
\end{align}
This formulation fundamentally alters the computational graph. Detailed derivation and convolution formulation can be found in the Appendix~\ref{appendix:discrete}.

\vspace{5pt}\noindent\textbf{Data-Dependent Parameters.}
Recent SSM advances (\eg, Mamba) emphasize input-conditioned dynamics, where transition/measurement vectors and step sizes adapt to the current input. Furthermore, models need to adaptively learn the importance of temporal variations based on the data. We follow this principle and make the discrete SSM parameters data dependent: 
\begin{equation}
\begin{aligned}
B[t,c] &= f^B_\theta\!\big(x[t,c]\big), \\
C[t,c] &= f^C_\theta\!\big(x[t,c]\big), \\
\Delta[t] &= f^\delta_\theta\!\big(x[t,:]\big).
\end{aligned}
\label{eq15}
\end{equation}
Here, $f_\theta$ is a pointwise map applied per variable and time; their outputs have the same shape as the corresponding SSM vectors.

\vspace{5pt}\noindent\textbf{2D Selective Scan.} 
Although our model is a 2D SSM, it exploits the variable-invariant construction to avoid an explicit vertical scan. At each time $t$, we compute the permutation-invariant summary $\psi[t]$ and then perform a hybrid scan:
\begin{itemize}
    \item a 1D selective scan along the time axis (causal, efficient, input-conditioned via $\Delta[t], B[t,c], C[t,c]$,
    \item coupled with global pooling on the variable axis to transmit cross-variable information.
\end{itemize}
Note that while conventional 2D SSMs require a nested loop of complexity $\mathcal{O}(T \times C)$ due to sequential dependencies, our model decouples the variable axis. The global term $\psi[t]$ is computed via parallel reduction, allowing the temporal scan to proceed independently for each variable. This reduces the effective dependency depth along the variable axis to $\mathcal{O}(1)$.

\begin{algorithm}[th]
\caption{Variable-Invariant 2D SSM (Forward)}
\label{alg:vi2dssm}
\DontPrintSemicolon
\SetKwInOut{Input}{Input}
\SetKwInOut{Output}{Output}
\Input{$X \in \mathbb{R}^{C \times T}$ \quad ($C$ variables over $T$ steps)}
\Input{$\phi$ \quad (Permutation-invariant aggregator)}
\Output{$Y$}
Initialize $h_h$, $h_v$\;
\For{$t \gets 1$ \KwTo $T$}{ \tcp{single temporal scan}
  $\psi[t] \gets \phi\!\left(\{W_v\,z[t-1,c]\}_{c=1}^C\right)
$\;
  \For{$c \gets 1$ \KwTo $C$}{
    $h_h[t,c] \gets \bar{A}_h h_h[t-1,c] + \bar{B}^x_{h} X[t,c] + \bar{B}^\psi_{h} \psi[t]$\;
    $h_v[t,c] \gets \bar{A}_v h_v[t-1,c] + \bar{A}_{vh} h_h[t,c] + \bar{B}^x_{v} X[t,c] + \bar{B}^\psi_{v} \psi[t]$\;
  }
}
$Y \gets C \,[\,h_h,\, h_v\,]$\;
\Return $Y$\;
\end{algorithm}

\subsection{Variable-Invariant 2D Mamba Architecture}
Building on the theoretical foundation of the VI 2D SSM, we propose the Variable-Invariant 2D Mamba architecture. To capture the multi-scale nature of time series, we integrate three complemenatry pathways: multi-dcale temporal pathways and spectral-domain pathway.
\subsubsection{Multi-Scale Temporal Pathways}  
Time-series data inherently exhibit structure at multiple scales. Depending on the task and domain, informative signals may range from local, fine-grained dependencies between adjacent time points to long-term, global trends spanning extended horizons. To capture such patterns, recent research has emphasized multi-scale temporal modeling as a central design principle~\citep{behrouz2024chimera, karadag2025ms}.  

In our framework, multi-scale dynamics are realized by adjusting the discretization parameter $\Delta$ during SSM discretization. Following prior work~\citep{behrouz2024chimera, karadag2025ms}, $\Delta$ can be interpreted as the effective time resolution or sampling rate in continuous-time formulations.
We instantiate two parallel VI 2D SSM blocks with distinct discretization step sizes.
\begin{itemize}
    \item Long-term Branch ($\Delta_l$): A coarse-grained branch (large $\Delta$) that filters out high-frequency noise to capture global trends and seasonalities.
    \item Short-term Branch ($\Delta_s$): A fine-grained branch (small $\Delta$) focused on capturing rapid fluctuations and transient events.
\end{itemize}
This design mimics a continuous wavelet transform, resolving features at different temporal resolutions. Larger $\Delta$ values correspond to slower state updates, effectively capturing long-range trends, whereas smaller $\Delta$ values yield faster dynamics suited for modeling rapid fluctuations. 

\subsubsection{Spectral-Domain Pathway}  
For many time-series modalities, the frequency domain plays a central role in characterizing the underlying signal~\citep{qiu2025comprehensive}. Frequency-domain features complement time-domain representations: low-frequency components often capture global, slowly varying dynamics, whereas high-frequency components reveal localized, transient variations. To leverage this complementary structure, we introduce a frequency-domain pathway into our architecture.  

Concretely, the input sequence is transformed into the frequency domain using the discrete Fourier transform along the temporal axis. 
For real-valued signals, only the non-redundant spectral components are retained, and the real and imaginary parts are concatenated to form the frequency representation. 
This yields a tensor with the same dimensionality as the original input, which is then processed by the proposed 2D SSM operating along the frequency axis instead of the temporal axis.


\vspace{5pt}\noindent\textbf{Interpretation of SSM in the Frequency Domain.}  
In this setting, the semantics of the SSM differ from the temporal case. Rather than modeling evolution over time, the states evolve across frequency bands, capturing dependencies between variables across spectral ranges. Importantly, the update operates as a continuous-time SSM applied to the frequency axis, where the discretization step size $\Delta_f$ governs how densely the underlying spectral ODE is sampled. Unlike time, frequency does not exhibit causality; it spans from low to high values, with most signal energy concentrated at low frequencies and increasingly sparse but informative oscillatory patterns at higher frequencies.

This spectral imbalance imposes a numerical challenge. If $\Delta_f$ is large, the discretization becomes too coarse; the exponential term in the SSM update inaccurately scales higher-frequency modes, causing instability, aliasing, and over-attenuation of small but critical spectral components. Setting $\Delta_f$ sufficiently small stabilizes the SSM dynamics, suppresses aliasing, and increases the effective resolution in the high-frequency range, ensuring that transient or oscillatory structure is preserved despite low energy levels\footnote{In practice, we find that values of $\Delta_f$ within $[0.001, 0.01]$ offer a robust trade-off between spectral fidelity and numerical stability.}.

\subsubsection{Adaptive Gating and Fusion}  
Our final architecture integrates information from three complementary pathways: the long-term temporal branch ($\Delta_l$), the short-term temporal branch ($\Delta_s$), and the frequency-domain branch ($\Delta_f$). $$h_{\text{fused}} = \text{Gate}(h_{\text{long}}, h_{\text{short}}, h_{\text{spec}})$$ This learnable gating mechanism allows the model to dynamically weight the importance of trend, transient, and spectral features depending on the input instance, providing a robust representation for downstream tasks.


\begin{table*}[ht]
\centering
\caption{Experimental results of long-term time-series forecasting. Best results are highlighted in \textbf{bold}, and second-best results are \underline{underlined}.}
\label{exp:long_term_forecasting}
\setlength{\tabcolsep}{2pt}
\renewcommand{\arraystretch}{1.2}
\resizebox{\textwidth}{!}{
\begin{tabular}{l|cccccccccccccccccccccc}
\toprule
\multirow{2}{*}{Dataset} & \multicolumn{2}{c}{Ours} & \multicolumn{2}{c}{Chimera} & \multicolumn{2}{c}{TimePro} & \multicolumn{2}{c}{Simba} & \multicolumn{2}{c}{TCN} & \multicolumn{2}{c}{iTransformer} & \multicolumn{2}{c}{RLinear} & \multicolumn{2}{c}{PatchTST} & \multicolumn{2}{c}{Crossformer} & \multicolumn{2}{c}{TiDE} & \multicolumn{2}{c}{TimesNet} \\
\cmidrule{2-23}
 & MSE & MAE & MSE & MAE & MSE & MAE & MSE & MAE & MSE & MAE & MSE & MAE & MSE & MAE & MSE & MAE & MSE & MAE & MSE & MAE & MSE & MAE \\
\midrule
ETTm1 & \textbf{0.345} & \textbf{0.378} & 0.355 & \underline{0.381} & 0.391 & 0.400 & 0.383 & 0.396 & \underline{0.351} & \underline{0.381} & 0.407 & 0.410 & 0.414 & 0.407 & 0.387 & 0.400 & 0.513 & 0.496 & 0.419 & 0.419 & 0.400 & 0.406\\
\midrule
ETTm2 & 0.258 & \underline{0.317} & \textbf{0.252} & \underline{0.317}  & 0.281 & 0.326 & 0.271 & 0.327 & \underline{0.253} & \textbf{0.314} & 0.288 & 0.332 & 0.286 & 0.327 & 0.281 & 0.326 & 0.757 & 0.610 & 0.358 & 0.404 & 0.291 & 0.333\\
\midrule
ETTh1 & \textbf{0.397} & \textbf{0.419}  & 0.408 & 0.425 & 0.438 & 0.438 & 0.441 & 0.432 & \underline{0.404} & \underline{0.420} & 0.454 & 0.447 & 0.446 & 0.434 & 0.469 & 0.454 & 0.529 & 0.522 & 0.541 & 0.507 & 0.458 & 0.450\\
\midrule
ETTh2 & \textbf{0.313} & \textbf{0.372} & \underline{0.321} & \underline{0.377} 
 & 0.377 & 0.403 & 0.361 & 0.391 & 0.322 & 0.379 & 0.383 & 0.407 & 0.374 & 0.398 & 0.387 & 0.407 & 0.942 & 0.684 & 0.611 & 0.550 & 0.414 & 0.427\\
\midrule
ECL & 0.158 & 0.254 & \textbf{0.154} & \textbf{0.249}  & 0.169 & 0.262 & 0.185 & 0.274 & \underline{0.156} & \underline{0.253} & 0.178 & 0.270 & 0.219 & 0.298 & 0.205 & 0.290 & 0.244 & 0.334 & 0.251 & 0.344 & 0.192 & 0.295\\
\midrule
Exchange & \underline{0.311} & 0.375 & \underline{0.311} & \textbf{0.358} & 0.352 & 0.399 & - & - & \textbf{0.302} & \underline{0.366} & 0.360 & 0.403 & 0.378 & 0.417 & 0.367 & 0.404 & 0.940 & 0.707 & 0.370 & 0.413 & 0.416 & 0.443 \\
\midrule
Traffic & \textbf{0.398} & \underline{0.276} & \underline{0.403} & 0.286 & - & -&  0.493 & 0.291 & \textbf{0.398} & \textbf{0.270} & 0.428 & 0.282 & 0.626 & 0.378 & 0.481 & 0.304 & 0.550 & 0.304 & 0.760 & 0.473 & 0.620 & 0.336 \\
\midrule
Weather & 0.227 & \underline{0.264} & \textbf{0.219} & \textbf{0.258} & 0.251 & 0.276 & 0.255 & 0.280 & \underline{0.224} & \underline{0.264} & 0.258 & 0.278 & 0.272 & 0.291 & 0.259 & 0.281 & 0.259 & 0.315 & 0.271 & 0.320 & 0.259 & 0.287\\
\midrule\midrule
1\textsuperscript{st} Count & \textbf{4} & \textbf{3} & 3 & 3 &
0 & 0 & 0 & 0 & 2 & 2 & 0 & 0 & 0 & 0 & 0 & 0 & 0 & 0 & 0 & 0 & 0 & 0 \\
\bottomrule
\end{tabular}
}
\end{table*}


\section{Experiments}
\label{experiments}
To evaluate the validity and generality of the proposed method, we conducted experiments across four representative tasks: long-term forecasting, short-term forecasting, classification, and anomaly detection. The proposed method was extensively compared against state-of-the-art baselines, including Transformer-based approaches~\citep{liu2023itransformer, nie2022time, zhang2023crossformer} and Mamba-based method~\citep{patro2024simba, behrouz2024chimera}. Detailed experimental settings and additional results are provided in the Appendix~\ref{appendix:exp} and Github\footnote{https://github.com/ku-milab/Variable-Invariant-2D-SSM}.

\subsection{Long-term Forecasting}
\vspace{5pt}\noindent\textbf{Settings.}
We conducted experiments on eight benchmark datasets commonly used in time-series data forecasting tasks, including the Weather, Traffic, Electricity, Exchange, and 4 ETT datasets~\citep{zhou2021informer}. Baseline results were obtained from the corresponding literature to ensure consistency of comparison, and all models were evaluated using mean squared error (MSE) and mean absolute error (MAE). For brevity, we report the average results in the main text, while detailed results and extended baseline comparisons are provided in Table~\ref{appendix:long_term}.

\vspace{5pt}\noindent\textbf{Results.}
Table~\ref{exp:long_term_forecasting} highlights the effectiveness of the proposed method. It achieves the lowest MSE on four of the eight datasets and the lowest MAE on three, yielding the best overall performance among all baselines. Compared to Transformer-based methods, our approach consistently outperforms across all datasets, and it also surpasses 1D SSM variants~\citep{patro2024simba, matimepro}. Against Chimera~\citep{behrouz2024chimera}, a 2D SSM model, the proposed method achieves superior results on most datasets and remains competitive on the others, further demonstrating its robustness and effectiveness.

\begin{table*}[ht]
\centering
\caption{Experimental results of short-term time-series forecasting in the M4 dataset. Best results are highlighted in \textbf{bold}, and second-best results are \underline{underlined}.}
\label{exp:short_term_forecasting}
\renewcommand{\arraystretch}{1.2}
\resizebox{\textwidth}{!}{
\begin{tabular}{c|ccccccccccc}
\toprule
 Weighted Avg. & Ours & Chimera & ModernTCN & PatchTST & TimesNet & N-HiTS & N-BEATS & ETSformer & LightTS & DLinear & FEDformer \\
\midrule 
 SMAPE & \underline{11.686} & \textbf{11.618} & 11.698 & 11.807 & 11.829 & 11.927 & 11.851 & 14.718 & 13.525 & 13.639 & 12.840 \\
MASE & \underline{1.549} & \textbf{1.528} &  1.556 &  1.590 &  1.585 &  1.613 &  1.599 &  2.408 &  2.111 &  2.095 &  1.701 \\
OWA & \underline{0.832} &  \textbf{0.827} &  0.838 &  0.851 &  0.851 &  0.861 &  0.855 &  1.172 &  1.051 &  1.051 &  0.918 \\
\bottomrule
\end{tabular}
}
\end{table*}

\subsection{Short-term Forecasting}
\vspace{5pt}\noindent\textbf{Settings.}
Short-term forecasting was evaluated on the widely adopted M4 dataset~\citep{makridakis2018m4}. Performance was assessed using three standard metrics—Symmetric Mean Absolute Percentage Error (SMAPE), Mean Absolute Scaled Error (MASE), and Overall Weighted Average (OWA)—with baseline results drawn directly from the original literature to ensure comparability. For clarity, we report only the weighted average results in the main text, while the complete set of results is deferred to Table~\ref{appendix:short_term}.

\vspace{5pt}\noindent\textbf{Results.}
Table~\ref{exp:short_term_forecasting} reports the results on the M4 dataset. While the proposed method did not obtain the best score, it achieved the second-best performance, indicating its ability to effectively capture short-term patterns. Its slightly lower performance relative to Chimera is attributable to the dataset’s single-channel setting, where the advantages of variable-invariant modeling are less pronounced compared to Chimera’s inherent 2D formulation. Nonetheless, our method surpasses other state-of-the-art baselines, confirming its competitiveness in short-term forecasting.






\begin{figure}[t]
    \includegraphics[width=\linewidth]{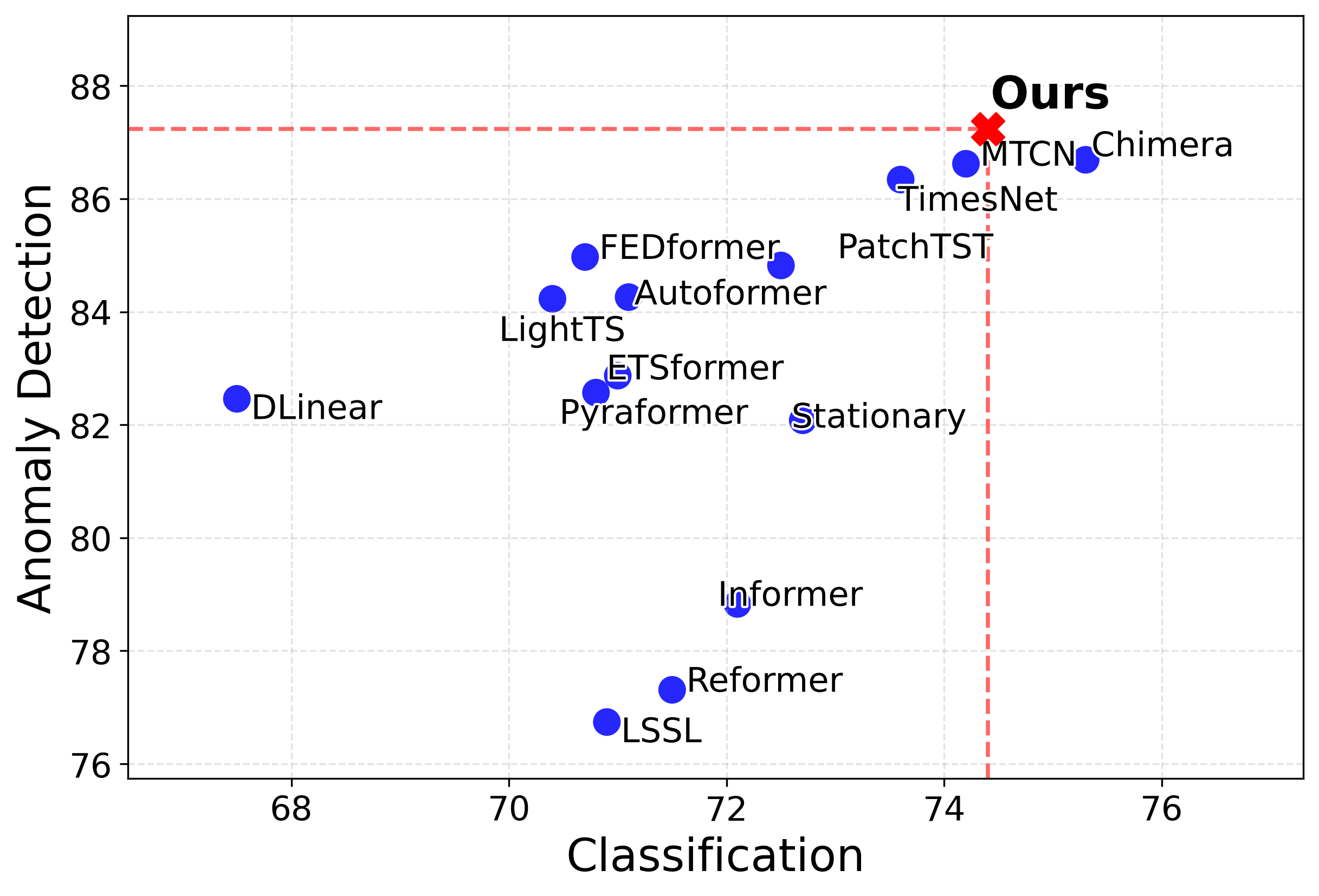}
    \caption{Results of classification (accuracy) and anomaly detection (F1 score).}
    \label{fig:cls_ad}
\end{figure}

\subsection{Classification \& Anomaly Detection}

\vspace{5pt}\noindent\textbf{Settings.}
For classification and anomaly detection, we benchmarked the proposed method on ten datasets from the UEA archive~\citep{bagnall2018uea} and five widely adopted anomaly detection datasets: SMD~\citep{su2019robust}, SWaT~\citep{mathur2016swat}, PSM~\citep{abdulaal2021practical}, MSL, and SMAP~\citep{hundman2018detecting}. Classification performance was assessed using overall accuracy, while anomaly detection was evaluated with precision, recall, and F1 score. Figure~\ref{fig:cls_ad} presents the aggregated results, reporting average accuracy (for classification) and average F1 score (for anomaly detection). Comprehensive results with detailed per-dataset comparisons are provided in the Table~\ref{appendix:classification} and~\ref{appendix:anomaly_detection}.


\vspace{5pt}\noindent\textbf{Results.}
Figure~\ref{fig:cls_ad} summarizes the overall results. Our proposed method achieves the best performance in anomaly detection, and while it falls slightly short of Chimera in classification, it remains superior to other baselines. This difference can be attributed to the nature of the tasks: anomaly detection strongly benefits from permutation-invariant modeling, as rare deviations often manifest across variable interactions without being tied to a fixed ordering. In contrast, classification tasks in the UEA datasets often involve structured temporal features with relatively limited variable dimensionality, where Chimera’s explicit sequential modeling of the variable axis may still provide an advantage. Although classification accuracy is marginally lower than Chimera, our method delivers competitive results at substantially lower computational cost, making it an attractive and efficient alternative—particularly given its clear gains in anomaly detection.

\begin{figure}[t]
    \includegraphics[width=1\linewidth]{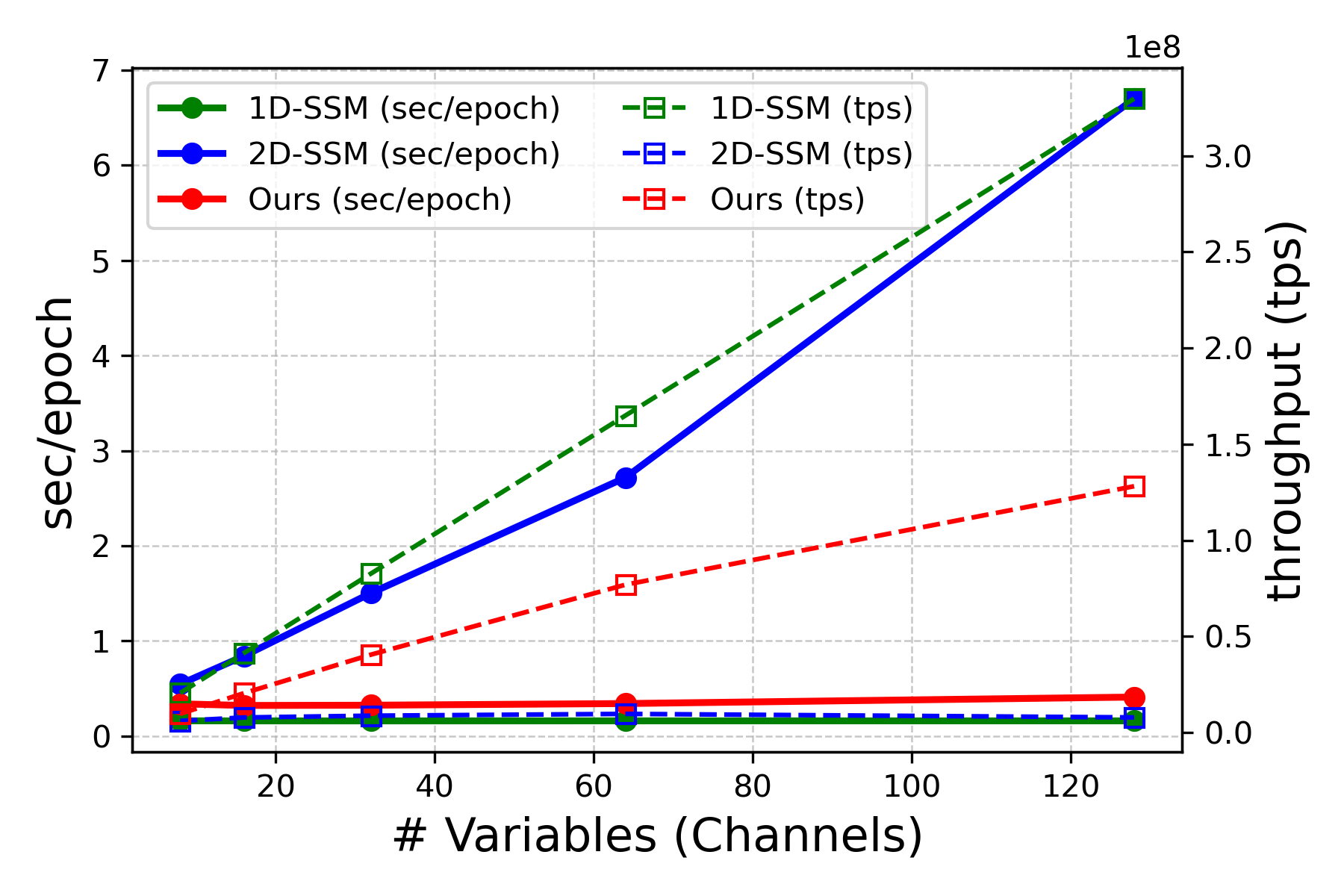}
    \caption{Efficiency analysis with respect to the number of variables. The left $y$-axis reports the training time in sec/epoch, while the right $y$-axis reports throughput in transactions per second.}
    \label{fig:efficiency}
\end{figure}

\section{Analysis}
To evaluate the superiority and efficiency of the proposed method, we conducted three complementary analyses. Specifically, we analyzed (\romannumeral 1) how efficient the proposed 2D SSM is compared to existing 2D SSMs, (\romannumeral 2) how well our proposed method performs in case studies where inter-variable relationships are important, and (\romannumeral 3) how each module of the proposed Mamba contributes to the overall performance.

\subsection{Efficiency of the Proposed 2D SSM}
To assess the computational efficiency of our formulation, we compared it against a 1D SSM and the 2D SSM~\citep{behrouz2024chimera}. For a fair comparison, all models shared identical configurations, differing only in the SSM component. We measured training efficiency using sec/epoch (lower is better) and throughput (higher is better), while varying the number of variables at a fixed sequence length of 256.

As shown in Figure~\ref{fig:efficiency}, our method achieves efficiency close to that of the 1D SSM, with only a slight overhead attributable to global variable aggregation. By contrast, conventional 2D SSM exhibits a sharp decline in efficiency as the number of variables increases, highlighting the scalability bottleneck introduced by sequential scans. These results confirm that the proposed formulation retains the efficiency properties discussed in Section~\ref{2dssm} while eliminating the inefficiencies of conventional 2D SSMs.

To clarify computational efficiency, we provide a quantitative comparison of FLOPs and peak GPU memory across representative transformer- and SSM-based baselines using their official implementations. Our model requires only a single temporal SSM scan plus a parallel pooling step, yielding substantially lower FLOPs than methods that recursively evolve states along both axes. A full derivation and empirical results are reported in the Section~\ref{analysis:eff}.

\begin{table}[th]
\centering
\caption{Controlled simulation for variable-ordering robustness.}
\label{exp:simulation}
\begin{tabular}{l|c|c|c}
\toprule
Setting & sec/epoch & MAE (std) & MAPE (std) \\
\midrule
Ours & \multirow{2}{*}{6.1s} & 0.093 (0.000) & 2.007 (0.000)\\
$+$ \textit{Permutation} & & 0.093 (0.000) & 2.063 (0.050) \\
\midrule
2D-SSM & \multirow{2}{*}{23.4s} & 0.093 (0.000) & 1.954 (0.000)\\
$+$ \textit{Permutation} & & 0.094 (0.001) & 2.341 (0.130) \\
\bottomrule
\end{tabular}
\end{table}

\begin{table}[th]
\centering
\caption{Controlled simulation for $C$-scaling.}
\label{exp:simulation_C}
\begin{tabular}{c|c|c|c|c}
\toprule
Method & \# C & sec/epoch & MAE & MAPE \\
\midrule
\multirow{5}{*}{Ours} & 16 & 6.2s & 0.106 & 2.465 \\
& 32 & 6.1s & 0.098 & 2.146 \\
& 64 & 6.1s & 0.093 & 2.007 \\
& 128 & 6.4s & 0.096 & 2.313 \\
& 256 & 6.3s & 0.097 & 2.663 \\
\midrule
\multirow{5}{*}{2D-SSM} & 16 & 13.4s & 0.105 & 2.603 \\
& 32 & 19.4s & 0.098 & 2.101 \\
& 64 & 23.4s & 0.093 & 1.954 \\
& 128 & 51.0s & 0.096 & 2.489 \\
& 256 & 88.8s & 0.097 & 2.779 \\
\bottomrule
\end{tabular}
\end{table}

\subsection{Controlled Simulation}
To evaluate accuracy in settings where inter-variable relationships are critical, we constructed a controlled simulation using a VAR(1) process defined on a Watts–Strogatz small-world graph~\citep{watts1998strogatz}. The dataset comprised 64 variables over 1000 time steps, and forecasting performance was assessed using MAE and MAPE (Table~\ref{exp:simulation}). To further investigate scalability with respect to the number of variables, we performed a controlled $C$-scaling experiment by varying $C \in \{16, 32, 64, 128, 256\}$ and measuring both performance and computational efficiency (Table~\ref{exp:simulation_C}).

Under the baseline setting (no permutation), the proposed method and a conventional 2D-SSM achieved comparable MAE/MAPE, indicating that both formulations can capture the underlying dynamics when a favorable variable ordering is provided. Nevertheless, our method was approximately 3.8$\times$ faster per epoch, reflecting the efficiency gain from removing the variable-axis scan.

To assess robustness to variable ordering, we repeated training and evaluation over 10 random permutations of the variable indices. The proposed method exhibited negligible performance drift across permutations (low variance), consistent with its permutation-invariant design via the global summary $\psi$. In contrast, the conventional 2D-SSM showed significantly larger performance variance, revealing sensitivity to the imposed ordering. This suggests that while 2D-SSM can achieve competitive accuracy under favorable indexings, its performance degrades and variability increases under re-orderings.

Beyond permutation robustness, the proposed formulation also exhibits favorable scaling behavior as the number of variables increases. While conventional 2D-SSMs require a recursive pass along the variable axis, our model replaces this sequence of updates with a single parallel pooling step, whose cost scales linearly in theory but executes as a fully batched GPU operation in practice. Both methods achieved comparable MAE/MAPE at all dimensionalities, confirming that eliminating variable-axis recurrence does not compromise expressiveness. However, while the runtime of standard 2D-SSM grew almost linearly with $C$, our method maintained nearly constant training time, demonstrating that the benefit of variance-invariant pooling intensifies as dimensionality grows. This highlights that the proposed formulation is particularly advantageous in high-dimensional multivariate systems where cross-variable interactions coexist with scalability demands.

\begin{table}[th]
\centering
\caption{Ablation study on ETTh1 and ETTm1. \textbf{Case \MakeUppercase{\romannumeral 1}}: branch removal; \textbf{Case \MakeUppercase{\romannumeral 2}}: different $\Delta_f$ ranges in the frequency domain.}
\label{exp:ablation_study}
\begin{tabular}{c|l|cc|cc}
\toprule
 \multicolumn{2}{c|}{\multirow{2}{*}{Setting}} & \multicolumn{2}{c|}{ETTh1} & \multicolumn{2}{c}{ETTm1}\\
\cmidrule{3-6}
\multicolumn{2}{c|}{} & MSE & MAE & MSE & MAE\\
\midrule
\multirow{3}{*}{\MakeUppercase{\romannumeral 1}} & w/o short & 0.433 & 0.440 & 0.349 & 0.381 \\
& w/o long & 0.431 & 0.439 & 0.350 & 0.381 \\
& w/o freq. & 0.425 & 0.437 & 0.348 & 0.381 \\
\midrule
\MakeUppercase{\romannumeral 2} & $[0.1,0.5]$ & 0.406 & 0.426 & 0.349 & 0.380 \\
($\Delta_f$) & $[0.01,0.05]$ & 0.404 & 0.423 & 0.349 & 0.380 \\
\midrule 
\MakeUppercase{\romannumeral 3} & Mean & 0.397 & 0.419 & 0.345 & 0.378 \\
($\psi$) & Attention & 0.409 & 0.426 & 0.347 & 0.381 \\
\midrule\midrule
\multicolumn{2}{c|}{Ours} & 0.397 & 0.419 & 0.345 & 0.378 \\
\bottomrule
\end{tabular}
\end{table}

\subsection{Ablation Study}
In Table~\ref{exp:ablation_study}, we systematically evaluate the contribution of each component of the proposed Mamba model through three ablation cases: (Case \MakeUppercase{\romannumeral 1}) analysis of individual branches, (Case \MakeUppercase{\romannumeral 2}) assessment of frequency-domain resolution under varying $\Delta_f$ settings, and (Case \MakeUppercase{\romannumeral 3}) comparison of permutation-invariant aggregation functions $\psi$.

In Case \MakeUppercase{\romannumeral 1}, removing any branch led to performance degradation, confirming that all modules contribute meaningfully. Among them, the short-term and long-term pathways have the most significant impact, while the frequency branch plays a complementary role. Case \MakeUppercase{\romannumeral 2} further highlights this complementarity: although the best results are obtained when all branches are combined, performance gradually improves as $\Delta_f$ decreases, indicating that finer resolution in the high-frequency region enhances modeling capacity. In Case \MakeUppercase{\romannumeral 3}, mean pooling consistently outperforms attention-based pooling across both datasets. This supports our design choice: forecasting tasks benefit more from stable, low-variance, parameter-free aggregation, with minimal computational overhead. This ablation confirms that the default choice of mean pooling for forecasting tasks is both empirically justified and computationally optimal. A detailed discussion of task-dependent aggregation choices and their computational properties is provided in the Appendix~\ref{appendix:aggregation}.

\begin{table}[th]
\centering
\caption{Measured peak GPU memory and forward FLOPs.}
\label{analysis:flops}
\begin{tabular}{l|c|c}
\toprule
Method & Memory (MB) & FLOPs (G)\\
\midrule
Ours & 546.62 & 11.99 \\ 
TimePro & 164.35 & 35.04 \\ 
S-Mamba & 270.76 & 96.13 \\ 
PatchTST & 675.26 & 85.64 \\ 
iTransformer & 1155.14 & 73.20 \\ 
\bottomrule
\end{tabular}
\end{table}

\subsection{Computational Efficiency Analysis}
\label{analysis:eff}
We provide a quantitative comparison of computational cost across representative transformer- and SSM-based multivariate forecasting models in Table~\ref{analysis:flops}. All methods were evaluated under the same conditions: ECL dataset, input window length $L = 96$, prediction window $H = 720$, and batch size 16. We report measurements only for models with official code enabling reproducible FLOPs and peak memory profiling.

A conventional 2D State Space Model processes dynamics along both temporal and variable axes. Let $L$ be the sequence length, $C$ the number of variables, and $d$ the hidden dimension. The standard cost is:
\begin{align}
    \underbrace{\mathcal{O}(LCd)}_{\text{temporal SSM scan}} + \underbrace{\mathcal{O}(LCd)}_{\text{variable-axis recursion}} = \mathcal{O}(2LCd).
\end{align}
The second term is sequential in $C$, restricting GPU parallelism.

We replace the recursive variable update with a single permutation-invariant pooled descriptor:
\begin{align}
    \underbrace{\mathcal{O}(LCd)}_{\text{temporal SSM scan}} + \underbrace{\mathcal{O}(Cd)}_{\text{Global pooling}} = \mathcal{O}(LCd) + \mathcal{O}(Cd).
\end{align}
Crucially, the additional $\mathcal{O}(Cd)$ term is computed in one fully parallel GPU kernel, rather than a sequential pass. This reduces practical computation to a single directional SSM scan, eliminating any dependence on variable ordering or recursive updates. Because our model does not evolve states along the variable axis, it avoids the second $\mathcal{O}(LCd)$ recurrence found in SSM baselines (TimePro, S-Mamba). The measured FLOPs are therefore markedly lower, despite additional frequency-domain modeling.

\section{Conclusion}


In this work, we introduced a variable-invariant two-dimensional state space model for multivariate time series. By replacing ordered recurrence along the variable axis with a global permutation-invariant aggregation, the proposed VI 2D SSM eliminates artificial ordering dependence while preserving global inter-variable coupling and enabling full parallelism.

Beyond architectural design, we provided a theoretical analysis of symmetry-constrained inter-variable dynamics. We showed that permutation equivariance uniquely determines the canonical form of admissible linear coupling and that the proposed formulation provides a parameterization consistent with this class through local self-dynamics and global pooled interaction. This symmetry-induced structure reduces computational dependency depth and simplifies stability analysis via a two-mode invariant decomposition.


Extensive experiments on forecasting, classification, and anomaly detection benchmarks demonstrate consistent improvements over state-of-the-art baselines in both accuracy and efficiency. These results highlight the proposed framework as a scalable, symmetry-grounded, and theoretically justified approach for modeling complex temporal–spectral dynamics in multivariate time series.

\section*{Acknowledgments}
This work was supported by Institute of Information \& communications Technology Planning \& Evaluation (IITP) grant funded by the Korea government(MSIT) (No. RS-2019-II190079, Artificial Intelligence Graduate School Program(Korea University), No. RS-2024-00457882, National AI Research Lab Project, and No. RS-2022-II220959 ((Part 2) Few-Shot Learning of Causal Inference in Vision and Language for Decision Making)).

\ifCLASSOPTIONcaptionsoff
  \newpage
\fi

\bibliographystyle{IEEEtran}
\bibliography{main}

\begin{IEEEbiography}[{\includegraphics[width=1in,height=1.25in,clip,keepaspectratio]{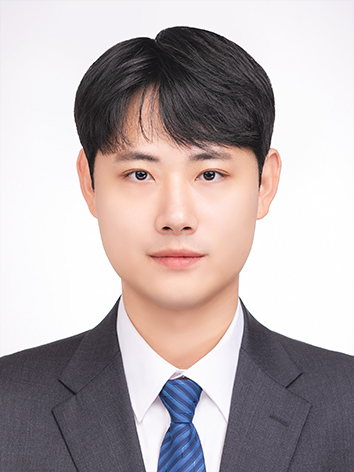}}]{Seungwoo Jeong}
	received the B.S. degree in Mathematics and Statistics from Hankuk University of Foreign Studies, Yongin, South Korea, in 2019. He is currently pursuing the Ph.D. degree with the Department of Artificial Intelligence, Korea University, Seoul, South Korea.
    
    His current research interests include time-series analysis, brain-computer interface, and geometric learning. 
\end{IEEEbiography}
\begin{IEEEbiography}[{\includegraphics[width=1in,height=1.25in,clip,keepaspectratio]{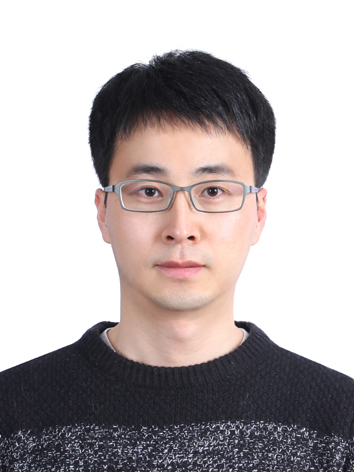}}]{Heung-Il Suk}
     is currently a Professor at the Department of Artificial Intelligence at Korea University. He was a Visiting Professor at the Department of Radiology at Duke University between 2022 and 2023.

     He was awarded a Kakao Faculty Fellowship from Kakao and a Young Researcher Award from the Korean Society for Human Brain Mapping (KHBM) in 2018 and 2019, respectively. His research interests include causal machine/deep learning, explainable AI, biomedical data analysis, and brain-computer interface.

     Dr. Suk serves as an Editorial Board Member for Clinical and Molecular Hepatology (Artificial Intelligence Sector), Electronics, Frontiers in Neuroscience, Frontiers in Radiology (Artificial Intelligence in Radiology), International Journal of Imaging Systems and Technology (IJIST), and a Program Committee or a Reviewer for NeurIPS, ICML, ICLR, AAAI, IJCAI, CVPR, MICCAI, AISTATS, \etc.
\end{IEEEbiography}

\clearpage

\appendices
\section{Variable-Invariant 2D SSM}
\label{appendix:discrete}

\subsection{Discretization of the 2D SSM}
The continuous proposed 2D SSM is given by:
\begin{equation*}
\begin{aligned}
\frac{d}{dt}
\begin{bmatrix}
h_h(t,c) \\
h_v(t,c)
\end{bmatrix}
&=
\begin{bmatrix}
A_h & 0 \\
A_{vh} & A_v
\end{bmatrix}
\begin{bmatrix}
h_h(t,c) \\
h_v(t,c)
\end{bmatrix}  \\
&\quad+
\begin{bmatrix}
A_{h\psi} \\
A_{v\psi}
\end{bmatrix}
\psi(t)
+
\begin{bmatrix}
B_h \\
B_v
\end{bmatrix}
x(t,c).
\end{aligned}
\end{equation*}
In matrix form:
\begin{align*}
\mathcal{A} = \begin{bmatrix} A_{h} & 0 \\ A_{vh} & A_{v} \end{bmatrix}, \quad \mathcal{B} = \begin{bmatrix} A_{h\psi} & B_h \\ A_{v\psi} & B_v \end{bmatrix}.
\end{align*}
The general solution of the continuous state equation is:
\begin{align*}
\mathbf{h}(t) = e^{\mathcal{A}(t-t_0)} \mathbf{h}(t_0) + \int_{t_0}^{t} e^{\mathcal{A}(t-\tau)} \mathcal{B} \mathbf{u}(\tau) d\tau,
\end{align*}
where $\mathbf{h}(t) = \begin{bmatrix} h_h(t, c) \\ h_v(t, c) \end{bmatrix}$ and $\mathbf{u}(t) = \begin{bmatrix} \psi(t) \\ x(t,c) \end{bmatrix}$.

Under the ZOH assumption, when $t_0 = k\Delta$ and $t = (k+1)\Delta$:
\begin{equation*}
\begin{aligned}
\mathbf{h}((k+1)\Delta)
&= e^{\mathcal{A}\Delta} \mathbf{h}(k\Delta) \\
&\quad + \int_{k\Delta}^{(k+1)\Delta}
e^{\mathcal{A}((k+1)\Delta-\tau)}
\mathcal{B}\,\mathbf{u}[k] \, d\tau.
\end{aligned}
\end{equation*}
Since $\mathbf{u}(\tau) = \mathbf{u}[k]$ is constant over the interval due to ZOH:
\begin{align*}
\mathbf{h}[k+1] &= e^{\mathcal{A}\Delta} \mathbf{h}[k] + \Big(\int_{k\Delta}^{(k+1)\Delta} e^{\mathcal{A}((k+1)\Delta-\tau)} \mathcal{B} d\tau\Big)\mathbf{u}[k]
\\ &= e^{\mathcal{A}\Delta} \mathbf{h}[k] + \Big(\int_{0}^{\Delta} e^{\mathcal{A}(\Delta-\lambda)} \mathcal{B} d\lambda\Big)\mathbf{u}[k]
\\ &= e^{\mathcal{A}\Delta} \mathbf{h}[k] + \Big(\int_{0}^{\Delta} e^{\mathcal{A}\nu} \mathcal{B} d\nu\Big) \mathbf{u}[k].
\end{align*}
Therefore, the discretized system is:
\begin{align*}
\mathbf{h}[k+1,c] &= \bar{\mathcal{A}} \mathbf{h}[k,c] + \bar{\mathcal{B}} \begin{bmatrix} \psi[k] \\ x[k,c] \end{bmatrix},
\end{align*}
where:
\begin{align*}
\bar{\mathcal{A}} &= e^{\mathcal{A}\Delta} = \exp\left(\begin{bmatrix} A_{h} & 0 \\ A_{vh} & A_{v} \end{bmatrix} \Delta\right), \\
\bar{\mathcal{B}} &= \Big(\int_{0}^{\Delta} e^{\mathcal{A}\tau} d\tau\Big) \mathcal{B}.
\end{align*}

\subsection{Convolution Formulation of Variable-Invariant 2D SSM}
Applying the recurrent rules in Eq. 11-13, we can write the output as: 
\begin{align*}
y[t, c] = \sum_{\tau}^{t-1}K_{\psi}[t-1-\tau]\psi[\tau] + \sum_{\tau}^{t-1}K_x[t-1-\tau]x[\tau, c],
\end{align*}
where global info kernel $K_\psi[k] = [C_h, C_v]\bar{\mathcal{A}}_k\bar{\mathcal{B}}[:, 0]$ and input kernel $K_x[k] = [C_h, C_v]\bar{\mathcal{A}}_k\bar{\mathcal{B}}[:, 1]$.

\section{Theoretical proof}
\subsection{Proposition 1}
\label{appendix:proof_p1}
\textit{
Let $\pi$ be any permutation of $\{1,\ldots, C\}$ and define 
$\psi_\pi(t) := \phi\!\left(\{W_v z(t,\pi(c))\}_{c=1}^C\right)$.
Assume (i) $W_v$ is variable-shared (independent of $c$), and (ii) $\phi$ is permutation-invariant on multisets.
Then $\psi_\pi(t) = \psi(t)$ for all $t$.
}
\begin{proof}
Define the variable-wise multiset
\begin{align*}
\mathcal{S}_t = \{W_v z(t,c)\}_{c=1}^C.
\end{align*}
By definition, $\psi(t) = \phi(\mathcal{S}_t)$.

Under permutation $\pi \in S_C$, since $z$ is permutation-equivariant
along the variable axis, we have
\begin{align*}
z^\pi(t,c) = z(t,\pi(c)).
\end{align*}
Hence, the permuted collection becomes
\begin{align*}
\mathcal{S}_t^\pi
= \{W_v z^\pi(t,c)\}_{c=1}^C
= \{W_v z(t,\pi(c))\}_{c=1}^C.
\end{align*}

Because $\pi$ is a bijection on $\{1,\ldots,C\}$,
$\mathcal{S}_t^\pi$ coincides with $\mathcal{S}_t$
as a multiset. Since $\phi$ is permutation-invariant,
\begin{align*}
\phi(\mathcal{S}_t^\pi) = \phi(\mathcal{S}_t).
\end{align*}
Therefore, $\psi_\pi(t) = \psi(t)$ for all $t$.
\end{proof}

\subsection{Proposition 2}
\label{appendix:proof_p2}
\textit{
Assume $A_h, A_v, A_{h\psi}, A_{v\psi}, A_{vh}, B_h, B_v$ are variable-shared (independent of $c$), and let $\psi(t)$ be defined as above with a permutation-invariant $\phi$ (Prop. 1). For any permutation $\pi$ of $\{1, \ldots, C\}$, consider the permuted inputs $x^\pi(t, c) :=x(t, \pi(c))$ and permuted initial states $h^\pi(0, c):=h(0, \pi(c))$. Then the unique solutions to Eq. 8-9 satisfy $h^\pi_h(t, c) = h_h(t, \pi(c))$, $h^\pi_v(t, c)=h_v(t, \pi(c))$ $\forall\; t, c$, i.e., the dynamics are permutation equivariant with respect to the variable axis.
}
\begin{proof}
    Let $h_h, h_v$ be a solution of the continuous dynamics
    \begin{equation*}
    \begin{split}
    \frac{\partial h_h(t, c)}{\partial t}
    &= A_h h_h(t, c) + A_{h\psi}\psi(t) \\
    &\quad + B_h x(t, c),
    \end{split}
    \end{equation*} 
    \begin{equation*}
    \begin{split}
    \frac{\partial h_v(t, c)}{\partial t}
    &= A_v h_v(t, c) + A_{v\psi}\psi(t) \\
    &\quad + A_{vh} h_h(t, c) + B_v x(t,c).
    \end{split}
    \end{equation*}
    with initial conditions $h_h(0, c), h_v(0, c)$. Here, the matrices $A_h, A_v, A_{h\psi}, A_{v\psi}, B_{vh}, B_h, B_v$ are \emph{variable-shared} (do not depend on $c$), and $\psi(t) = \phi(\{W_vh_v(t, c)\}^C_{c=1})$ with $\phi$ permutation-invariant (Prop. 1). Fix any permutation $\pi$ of $\{1,\ldots, C\}$ and let the permuted inputs $x^\pi(t, c):=x(t, \pi(c))$, and the permuted initial states $h^\pi_h(0, c):= h_h(0, \pi(c)), h_v(0, c):=h_v(0, \pi(c))$. Consider the $\tilde{h}_h(t, c):= h_h(t, \pi(c)), \tilde{h}_v(t, c) := h_v(t, \pi(c))$. We show that $(\tilde{h}_h, \tilde{h}_v)$ satisfy the permuted system driven by $x^\pi$ with initial conditions $h^\pi_h(0, \cdot), h_v^\pi(0, \cdot)$.
    
    \vspace{5pt}\noindent\textbf{$\psi$ is invariant under variable permutations.}
    
    By Prop. 1, for any $t$, $\psi^\pi(t):=\psi(t)$.
    Hence, the pooled descriptor used in the dynamics is unchanged by the permutation.

    \vspace{5pt}\noindent\textbf{$\tilde{h}_h, \tilde{h}_v$ satisfy the permuted ODEs.}

    Differentiating $\tilde{h}_h$ and using Eq.~8,
    \begin{align*}
        \frac{\partial}{\partial t}\tilde{h}_h(t, c) &= \frac{\partial}{\partial t}h_h(t, \pi(c))
        \\&= A_h h_h(t, \pi(c)) + A_{h\psi}\psi(t) + B_h x(t, \pi(c))
        \\&= A_h \tilde{h}_h(t, c) + A_{h\psi}\psi(t) + B_h x^\pi(t, c).
    \end{align*}
    Likewise, using Eq.~9,
    \begin{equation*}
    \begin{aligned}
    \frac{\partial}{\partial t}\tilde{h}_v(t, c)
    &= A_v \tilde{h}_v(t, c)
    + A_{v\psi}\psi(t) \\
    &\quad + A_{vh}\tilde{h}_h(t, c)
    + B_v x^\pi(t,c).
    \end{aligned}
    \end{equation*}
    By construction, $\tilde{h}_h(0, c) = h_h^\pi(0, c)$ and $\tilde{h_v}(0, c) = h^\pi_v(0, c)$. Thus, $(\tilde{h}_h, \tilde{h}_v)$ are the solution of the permuted system with input $x^\pi$ and the permuted initial states.
    
    \vspace{5pt}\noindent\textbf{Uniqueness}

    The systems are linear time-varying (through $\psi$) but globally Lipschitz in the states. Hence, the solutions are unique. Therefore, the unique solution $(h^\pi_h, h^\pi_v)$ of the permuted system coincides with $(\tilde{h}_h, \tilde{h}_v)$, \ie,
    \begin{align*}
        (h^\pi_h(t, c), h^\pi_v(t, c)) &=(\tilde{h}_h(t, c), \tilde{h}_v(t, c)) \\&= (h_h(t, \pi(c)), h_v(t, \pi(c))),\quad\forall t, c.
    \end{align*}
    
    This proves permutation \emph{equivariance} of the dynamics with respect to the variable axis.
\end{proof}

\subsection{Theorem 1}
\label{appendix:thm1}
\textit{
Let $M \in \mathbb{R}^{C\times C}$. Then, 
\begin{align*}
    MP_\pi = P_\pi M \quad \forall \pi \in S_C 
\end{align*}
if and only if 
\begin{align*}
    M = \alpha I_C + \beta \mathbf{11}^\top
\end{align*}
for some scalars $\alpha, \beta \in \mathbb{R}$.}

\textit{
\noindent In particular, any linear permutation-equivariant state coupling must take the form
\begin{align*}
    h_v(t+1, c) = \alpha h_v(t, c) + \beta \sum^C_{j=1} h_v(t, j) + N x (t, c).
\end{align*}
    }
\begin{proof}

\textbf{Sufficiency}

Let $M = \alpha I_C + \beta \mathbf{11}^\top$. Since permutation matrices satisfy $P_\pi I_C = I_CP_\pi$, and $P_\pi\mathbf{1} = \mathbf{1},$ $\mathbf{1}^\top P_\pi = \mathbf{1}^\top,$ it follows that 
    \begin{align*}
        P_\pi (\mathbf{11}^\top) = (P_\pi \mathbf{1})\mathbf{1}^\top = \mathbf{11}^\top = \mathbf{1}(\mathbf{1}^\top P_\pi) = (\mathbf{11}^\top)P_\pi.
    \end{align*}
Hence, both $I_c$ and $\mathbf{11}^\top$ commute with every permutation matrix, and therefore
\begin{align*}
    MP_\pi = P_\pi M\quad \forall \pi \in S_C.
\end{align*}

\vspace{5pt}\noindent\textbf{Necessity}

Assume $MP_\pi = P_\pi M$ $\forall \pi \in S_C$. Since the symmetric group is generated by transpositions, it suffices to consider permutation matrices corresponding to swaps of two indices $i \neq j$.

Let $P_{ij}$ denote the matrix that exchanges indices $i$ and $j$. Right-multiplication by $P_{ij}$ swaps columns $i$ and $j$, while left-multiplication swaps rows $i$ and $j$. Thus, the commutation relation $MP_{ij} = P_{ij}M$ implies that exchanging columns $i$ and $j$ yields the same matrix as exchanging rows $i$ and $j$.

From entrywise comparison, we obtain: 
\begin{enumerate}
    \item All diagonal entries are equal $M_{ii} = M_{jj}\quad \forall i,j$.
    \item All off-diagonal entries are equal $M_{ik} = M_{jk},\quad M_{ki} = M_{kj}\quad \forall i\neq j,\; k \notin \{i,j\}$.
\end{enumerate}
Hence there exist scalars $d$ and $o$ such that
\begin{align*}
    M_{ii} = d,\quad M_{ij} = o\quad (i\neq j).
\end{align*}
Therefore, 
\begin{align*}
    M &= o\mathbf{11}^\top + (d-o) I_C 
    \\&= \alpha I_C + \beta \mathbf{11}^\top
\end{align*}
where $\alpha = d-o$ and $\beta = o$.
\end{proof}

\subsection{Theorem 2}
\textit{
Consider a 2D SSM with $C$ variables.
\begin{enumerate}
    \item If inter-variable coupling is implemented via ordered recurrence (\eg, $h_v(t,c+1)$ depends on $h_v(t,c)$), then the update at each fixed time $t$ induces a dependency chain of length $C$ along the variable axis.
    \item In contrast, the proposed permutation-invariant aggregation
    $\psi(t)$ reduces the variable-axis dependency depth to $\mathcal{O}(1)$, since all variable-wise updates are conditionally independent given $\psi(t)$.
\end{enumerate}
}
\begin{proof}
We measure the variable-axis dependency depth as the length of the longest directed path in the computational graph at a fixed time step $t$.
\begin{enumerate}
    \item In an ordered recurrence, $h_v(t,c+1)$ explicitly depends on $h_v(t,c)$. Hence the computation graph forms a directed chain
\[
h_v(t,1) \rightarrow h_v(t,2) \rightarrow \cdots \rightarrow h_v(t,C),
\]
whose length is $C$.

\item In the proposed formulation, the global descriptor $\psi(t)$ is computed via a permutation-invariant aggregation over $\{z(t,c)\}_{c=1}^C$. This aggregation can be implemented using a parallel reduction, whose depth is independent of $C$ (\eg, $\mathcal{O}(1)$ under ideal parallelization or $\mathcal{O}(\log C)$ in tree reduction). Once $\psi(t)$ is computed, each $h_v(t+1,c)$ depends only on local quantities and $\psi(t)$, and thus all updates can be performed in parallel.
\end{enumerate}
Hence, the longest dependency path along the variable axis is constant (up to reduction depth), establishing $\mathcal{O}(1)$ variable-axis dependency depth.
\end{proof}

\subsection{Theorem 3}
\textit{If the continuous-time matrix $A_v$ is Hurwitz, i.e.,
\[
\Re(\lambda_i(A_v)) < 0 \quad \forall i,
\]
then for any $\Delta > 0$, the discretized matrix
\[
\bar A_v = e^{\Delta A_v}
\]
satisfies
\[
\rho(\bar A_v) < 1,
\]
where $\rho(\cdot)$ denotes the spectral radius.}

\begin{proof}
Let $\lambda_i$ be an eigenvalue of $A_v$. By the spectral mapping theorem, the eigenvalues of $\bar A_v = e^{\Delta A_v}$ are $e^{\Delta \lambda_i}$.
Since $\Re(\lambda_i) < 0$ for all $i$ and $\Delta>0$,
\[
|e^{\Delta \lambda_i}| = e^{\Delta \Re(\lambda_i)} < 1.
\]
Therefore, all eigenvalues of $\bar A_v$ lie strictly inside the unit circle, and hence $\rho(\bar A_v) < 1$.
\end{proof}

\section{Experiments Details}
\label{appendix:exp}
To assess the effectiveness of the proposed method, we conducted extensive comparisons against a broad range of state-of-the-art models, including recent Mamba-based approaches. Baseline results were obtained from the corresponding literature to ensure fairness and consistency. The comparison models include Chimera~\citep{behrouz2024chimera}, TimePro~\citep{matimepro}, Simba~\citep{patro2024simba}, TCN~\citep{luo2024moderntcn}, iTransformer~\citep{liu2023itransformer}, RLinear~\citep{li2023revisiting}, PatchTST~\citep{nie2022time}, Crossformer~\citep{zhang2023crossformer}, TiDE~\citep{das2023long}, TimesNet~\citep{wu2022timesnet}, DLinear~\citep{zeng2023transformers}, SCINet~\citep{liu2022scinet}, FEDformer~\citep{zhou2022fedformer}, Stationary~\citep{liu2022non}, N-HiTS~\citep{challu2023nhits}, N-BEATS~\citep{oreshkin2019n}, ETSformer~\citep{woo2022etsformer}, LightTS~\citep{zhang2207less}, Autoformer~\citep{wu2021autoformer}, Pyraformer~\citep{liupyraformer}, Informer~\citep{zhou2021informer}, Reformer~\citep{kitaev2020reformer}, LSTM~\citep{hochreiter1997long}, LSTNet~\citep{lai2018modeling}, LSSL~\citep{gu2021efficiently}, Flowformer~\citep{wu2022flowformer} and LogTrans~\citep{li2019enhancing}.

\begin{table*}[ht]
\centering
\caption{Data description for time series forecasting.}
\label{appendix:forecasting_data}
\begin{tabular}{l|c|c|c|c}
\toprule
\multirow{2}{*}{Dataset} & \multirow{2}{*}{Types} & Sample Number & Variable & \multirow{2}{*}{Prediction Length} \\
& & (\small train/validation/test) & Dimension & \\
\midrule
ETTh1, ETTh2 & \multirow{6}{*}{Long-term} & (8545/2881/2881) & 7 & \{96 192, 336, 720\} \\
ETTm1, ETTm2 &  & (34465/11521/11521) & 7 & \{96 192, 336, 720\} \\
Electricity & & (18317/2633/5261) & 321 & \{96 192, 336, 720\} \\
Traffic & & (12185/1757/3509) & 862 & \{96 192, 336, 720\} \\
Weather & & (36792/5271/10540) & 21 & \{96 192, 336, 720\} \\
Exchange & & (5120/665/1422) & 8 & \{96 192, 336, 720\} \\
\midrule
M4-Yearly & \multirow{6}{*}{Short-term} & (23000/0/23000) & 1 & 6 \\
M4-Quarterly & & (24000/0/24000) & 1 & 8 \\
M4-Monthly & & (48000/0/48000) & 1 & 18 \\
M4-Weekly & & (359/0/359) & 1 & 13 \\
M4-Daily & & (4227/0/4227) & 1 & 14 \\
M4-Hourly & & (414/0/414) & 1 & 48 \\
\bottomrule
\end{tabular}
\end{table*}

\subsection{Datasets}
\vspace{5pt}\noindent\textbf{Forecasting.}
We used eight benchmark datasets for long-term forecasting (Weather, Traffic, Electricity, Exchange, and four ETT datasets~\citep{zhou2021informer}) and the M4 dataset for short-term forecasting~\citep{makridakis2018m4}. Long-term datasets consist of single continuous sequences with samples generated by sliding windows, whereas M4 contains 100,000 heterogeneous series from diverse domains. Dataset descriptions are provided in Table~\ref{appendix:forecasting_data}.

\begin{table*}[ht]
\centering
\caption{Data description for time series classification and anomaly detection.}
\label{appendix:cls_ad_data}
\resizebox{\textwidth}{!}{
\begin{tabular}{l|c|c|c|c}
\toprule
\multirow{2}{*}{Dataset} & \multirow{2}{*}{Types} & Sample Number & Variable & Series Length (\small classification) \\ 
 & & (\small train/validation/test) & Dimension & Sliding Window Length (\small anomaly detection) \\ 
\midrule
EthanolConcentration & \multirow{10}{*}{Classification} & (261/-/263) & 3 & 1751 \\
FaceDetection & & (5890/-/3524) & 144 & 62 \\
Handwriting & & (150/-/850) & 3 & 152 \\
Heartbeat & & (204/-/205) & 61 & 405 \\
JapaneseVowels & & (270/-/370) & 12 & 29 \\
PEMS-SF & & (267/-/173) & 963 & 144 \\
SelfRegulationSCP1 & & (268/-/293) & 6 & 896 \\
SelfRegulationSCP2 & & (200/-/180) & 7 & 1152 \\
SpokenArabicDigits & & (6599/-/2199) & 13 & 93 \\ 
UWaveGestureLibrary & & (120/-/320) & 3 & 315 \\
\midrule
SMD & \multirow{4}{*}{Anomaly} & (566724/141681/708420) & 38 & 100 \\
MSL & \multirow{4}{*}{Detection} & (44653/11664/73729) & 55 & 100 \\
SMAP & & (108146/27037/427617) & 25 & 100 \\
SWaT & & (396000/99000/449919) & 51 & 100 \\ 
PSM & & (105984/26497/87841) & 25 & 100 \\
\bottomrule
\end{tabular}}
\end{table*}

\vspace{5pt}\noindent\textbf{Classification \& Anomaly Detection.}
For classification, we used ten benchmark datasets from the UEA archive~\citep{bagnall2018uea}, spanning diverse domains such as vision, audio, industry monitoring, and medical diagnosis, with most tasks involving around ten classes. For anomaly detection, we adopted five widely used benchmarks: SMD~\citep{su2019robust}, SWaT~\citep{mathur2016swat}, PSM~\citep{abdulaal2021practical}, MSL, and SMAP~\citep{hundman2018detecting}, covering domains such as server machines, spacecraft, and critical infrastructure. Dataset descriptions are summarized in Table~\ref{appendix:cls_ad_data}.

\begin{table*}[ht]
    \centering
    \caption{Computational analysis of permutation-invariant aggregation functions. The trade-off between representational expressiveness and memory/compute overhead motivates task-adaptive selection of $\psi$.}
    \label{analysis:gpu}
    \begin{tabular}{l|c|c|c|c}
    \toprule
    \multirow{2}{*}{Aggregation $\psi$} & Work & \multirow{2}{*}{Parallel Span} & Peak Memory & \multirow{2}{*}{Characteristics}
    \\
     & (per timestep) & & Overhead & \\
    \midrule 
    Mean & $\mathcal{O}(Cd)$ & $\mathcal{O}(\log C)$ & $\mathcal{O}(d)$ & Stable, low-variance descriptor \\
    Sum & $\mathcal{O}(Cd)$ & $\mathcal{O}(\log C)$ & $\mathcal{O}(d)$ & Same cost as mean, without normalization \\
    Attention & $\mathcal{O}(Cd^2)$ & $\mathcal{O}(\log C)$ & $\mathcal{O}(Cd)$ & Expressive, selectively emphasizes informative variables \\
    \bottomrule
    \end{tabular}
\end{table*}

\subsection{Aggregation Function and Computational Behavior}
\label{appendix:aggregation}
\vspace{5pt}\noindent\textbf{Choice of the Aggregation Function $\psi$.}
The proposed formulation allows any permutation-invariant aggregation operator $\psi$, such as mean/sum pooling or attention-based pooling. In practice, the default choice of $\psi$ is selected according to the inductive bias and computational structure of each task:
\begin{itemize}
    \item \textbf{Forecasting.} Multi-step forecasting tasks (\eg, ETTh1/ETTm1) involve strong shared seasonal components across variables and are highly sensitive to variance amplification. Mean pooling produces a low-variance global descriptor that suppresses channel-wise noise and stabilizes long-horizon prediction. This property leads to more robust extrapolation than attention-based pooling.
    \item \textbf{Classification \& Anomaly Detection.} In contrast, these tasks rely on discriminative, potentially sparse cross-variable interactions, where only a subset of channels may indicate class identity or abnormal behavior. Mean pooling may dilute these informative variables, whereas attention pooling selectively emphasizes salient channels, improving representation quality.
\end{itemize}

\vspace{5pt}\noindent\textbf{GPU Throughput and Memory Usage.}
To clarify the computational implications of different choices of $\psi$, Table~\ref{analysis:gpu} summarizes the work complexity, parallel span, and peak memory usage. Let $C$ denote the number of variables and $d$ denote the hidden dimension.

Mean/sum pooling require only a single accumulator vector and do not store intermediate key/value tensors, resulting in minimal peak memory. Attention pooling, however, stores per-variable projections and attention scores, causing memory to grow linearly with $C$. This explains why forecasting, which benefits from stability and scalability, is paired with mean pooling, whereas expressive tasks (classification/anomaly detection) are paired with attention-based pooling.

\subsection{Full Experimental Results}
\label{appendix:additional_exp}

The overall results for long- and short-term forecasting, classification, and anomaly detection are reported in Table~\ref{appendix:long_term}-\ref{appendix:anomaly_detection}.

\begin{table*}[ht]
\centering
\caption{Full results of long-term time-series forecasting.}
\label{appendix:long_term}
\setlength{\tabcolsep}{2pt}
\renewcommand{\arraystretch}{1.2}
\resizebox{\textwidth}{!}{
\begin{tabular}{c|c|cc|cc|cc|cc|cc|cc|cc|cc|cc|cc|cc|cc|cc|cc}
\toprule
\multicolumn{2}{c|}{} & \multicolumn{2}{c|}{Ours} & \multicolumn{2}{c|}{Chimera} & \multicolumn{2}{c|}{TimePro} & \multicolumn{2}{c|}{Simba} & \multicolumn{2}{c|}{TCN} & \multicolumn{2}{c|}{iTransformer} & \multicolumn{2}{c|}{RLinear} & \multicolumn{2}{c|}{PatchTST} & \multicolumn{2}{c|}{Crossformer} & \multicolumn{2}{c|}{TiDE} & \multicolumn{2}{c|}{TimesNet} & \multicolumn{2}{c|}{DLinear} & \multicolumn{2}{c|}{SCINet} & \multicolumn{2}{c|}{FEDformer} \\
\cmidrule{3-30}
\multicolumn{2}{c|}{} & MSE & MAE & MSE & MAE & MSE & MAE & MSE & MAE & MSE & MAE & MSE & MAE & MSE & MAE & MSE & MAE & MSE & MAE & MSE & MAE & MSE & MAE & MSE & MAE & MSE & MAE & MSE & MAE \\
\midrule
\multirow{5}{*}{\rotatebox{90}{ETTm1}} & 96 & \textbf{0.287} & \textbf{0.342} & 0.318 & 0.354 & 0.326 & 0.364 & 0.324 & 0.360 & 0.292 & 0.346 & 0.334 & 0.368 & 0.355 & 0.376 & 0.329 & 0.367 & 0.404 & 0.426 & 0.364 & 0.387 & 0.338 & 0.375 & 0.345 & 0.372 & 0.418 & 0.438 & 0.379 & 0.419 \\
& 192 & \textbf{0.326} & \textbf{0.367} & 0.331 & 0.369 & 0.367 & 0.383 & 0.363 & 0.382 & 0.332 & 0.368 & 0.377 & 0.391 & 0.391 & 0.392 & 0.367 & 0.385 & 0.450 & 0.451 & 0.398 & 0.404 & 0.374 & 0.387 & 0.380 & 0.389 & 0.439 & 0.450 & 0.426 & 0.441 \\
& 336 & \textbf{0.359} & \textbf{0.387} & 0.363 & 0.389 & 0.402 & 0.409 & 0.395 & 0.405 & 0.365 & 0.391 & 0.426 & 0.420 & 0.424 & 0.415 & 0.399 & 0.410 & 0.532 & 0.515 & 0.428 & 0.425 & 0.410 & 0.411 & 0.413 & 0.413 & 0.490 & 0.485 & 0.445 & 0.459 \\
& 720 & \textbf{0.409} & 0.416 & \textbf{0.409} & \textbf{0.415} & 0.469 & 0.446 & 0.451 & 0.437 & 0.416 & 0.417 & 0.491 & 0.459 & 0.487 & 0.450 & 0.454 & 0.439 & 0.666 & 0.589 & 0.487 & 0.461 & 0.478 & 0.450 & 0.474 & 0.453 & 0.595 & 0.550 & 0.543 & 0.490 \\
\cmidrule{2-30}
& Avg & \textbf{0.345} & \textbf{0.378} & 0.355 & \underline{0.381} & 0.391 & 0.400 & 0.383 & 0.396 & \underline{0.351} & \underline{0.381} & 0.407 & 0.410 & 0.414 & 0.407 & 0.387 & 0.400 & 0.513 & 0.496 & 0.419 & 0.419 & 0.400 & 0.406 & 0.403 & 0.407 & 0.485 & 0.481 & 0.448 & 0.452\\
\midrule
\multirow{5}{*}{\rotatebox{90}{ETTm2}} & 96 & 0.169 & 0.259 & 0.169 & 0.265 & 0.178 & 0.260 & 0.177 & 0.263 & \textbf{0.166} & \textbf{0.256} & 0.180 & 0.264 & 0.182 & 0.265 & 0.175 & 0.259 & 0.287 & 0.366 & 0.207 & 0.305 & 0.187 & 0.267 & 0.193 & 0.292 & 0.286 & 0.377 & 0.203 & 0.287 \\
& 192 & 0.223 & 0.295 & \textbf{0.221} & \textbf{0.290} & 0.242 & 0.303 & 0.245 & 0.306 & 0.222 & 0.293 & 0.250 & 0.309 & 0.246 & 0.304 & 0.241 & 0.302 & 0.414 & 0.492 & 0.290 & 0.364 & 0.249 & 0.309 & 0.284 & 0.362 & 0.399 & 0.445 & 0.269 & 0.328  \\
& 336 & 0.279 & 0.332 & 0.279 & 0.339 & 0.303 & 0.342 & 0.304 & 0.343 & \textbf{0.272} & \textbf{0.324} & 0.311 & 0.348 & 0.307 & 0.342 & 0.305 & 0.343 & 0.597 & 0.542 & 0.377 & 0.422 & 0.321 & 0.351 & 0.369 & 0.427 & 0.637 & 0.591 & 0.325 & 0.366 \\
& 720 & 0.362 & 0.385 & \textbf{0.342} & \textbf{0.376} & 0.400 & 0.399 & 0.400 & 0.399 & 0.351 & 0.381 & 0.412 & 0.407 & 0.407 & 0.398 & 0.402 & 0.400 & 1.730 & 1.042 & 0.558 & 0.524 & 0.408 & 0.403 & 0.554 & 0.522 & 0.960 & 0.735 & 0.421 & 0.415  \\
\cmidrule{2-30}
& Avg & 0.258 & 0.318 & \textbf{0.252} & \underline{0.317} & 0.281 & 0.326 & 0.271 & 0.327 & \underline{0.253} & \textbf{0.314} & 0.288 & 0.332 & 0.286 & 0.327 & 0.281 & 0.326 & 0.757 & 0.610 & 0.358 & 0.404 & 0.291 & 0.333 & 0.350 & 0.401 & 0.571 & 0.537 & 0.305 & 0.349\\
\midrule
\multirow{5}{*}{\rotatebox{90}{ETTh1}} & 96 & \textbf{0.365} & 0.394 & 0.366 & \textbf{0.392} & 0.375 & 0.398 & 0.379 & 0.395 & 0.368 & 0.394 & 0.386 & 0.405 & 0.386 & 0.395 & 0.414 & 0.419 & 0.423 & 0.448 & 0.479 & 0.464 & 0.384 & 0.402 & 0.386 & 0.400 & 0.654 & 0.599 & 0.376 & 0.419 \\
& 192 & 0.407 & 0.417 & \textbf{0.402} & 0.414 & 0.427 & 0.429 & 0.432 & 0.424 & 0.405 & \textbf{0.413} & 0.441 & 0.436 & 0.437 & 0.424 & 0.460 & 0.445 & 0.471 & 0.474 & 0.525 & 0.492 & 0.436 & 0.429 & 0.437 & 0.432 & 0.719 & 0.631 & 0.420 & 0.448 \\
& 336 & \textbf{0.385} & 0.416 & 0.406 & 0.419 & 0.472 & 0.450 & 0.473 & 0.443 & 0.391 & \textbf{0.412} & 0.487 & 0.458 & 0.479 & 0.446 & 0.501 & 0.466 & 0.570 & 0.546 & 0.565 & 0.515 & 0.491 & 0.469 & 0.481 & 0.459 & 0.778 & 0.659 & 0.459 & 0.465 \\
& 720 & \textbf{0.434} & \textbf{0.452} & 0.458 & 0.477 & 0.476 & 0.474 & 0.483 & 0.469 & 0.450 & 0.461 & 0.503 & 0.491 & 0.481 & 0.470 & 0.500 & 0.488 & 0.653 & 0.621 & 0.594 & 0.558 & 0.521 & 0.500 & 0.519 & 0.516 & 0.836 & 0.699 & 0.506 & 0.507 \\
\cmidrule{2-30}
& Avg & \textbf{0.397} & \textbf{0.419}  & 0.408 & 0.425 & 0.438 & 0.438 & 0.441 & 0.432 & \underline{0.404} & \underline{0.420} & 0.454 & 0.447 & 0.446 & 0.434 & 0.469 & 0.454 & 0.529 & 0.522 & 0.541 & 0.507 & 0.458 & 0.450 & 0.456 & 0.452 & 0.747 & 0.647 & 0.440 & 0.460 \\
\midrule
\multirow{5}{*}{\rotatebox{90}{ETTh2}} & 96 & \textbf{0.257} & 0.329 & 0.262 & \textbf{0.327} & 0.293 & 0.345 & 0.290 & 0.339 & 0.263 & 0.332 & 0.297 & 0.349 & 0.288 & 0.338 & 0.302 & 0.348 & 0.745 & 0.584 & 0.400 & 0.440 & 0.340 & 0.374 & 0.333 & 0.387 & 0.707 & 0.621 & 0.358 & 0.397 \\
& 192 & \textbf{0.312} & \textbf{0.367} & 0.320 & 0.372 & 0.367 & 0.394 & 0.373 & 0.390 & 0.320 & 0.374 & 0.380 & 0.400 & 0.374 & 0.390 & 0.388 & 0.400 & 0.877 & 0.656 & 0.528 & 0.509 & 0.402 & 0.414 & 0.477 & 0.476 & 0.860 & 0.689 & 0.429 & 0.439 \\
& 336 & \textbf{0.303} & \textbf{0.367} & 0.316 & 0.381 & 0.419 & 0.431 & 0.376 & 0.406 & 0.313 & 0.376 & 0.428 & 0.432 & 0.415 & 0.426 & 0.426 & 0.433 & 1.043 & 0.731 & 0.643 & 0.571 & 0.452 & 0.452 & 0.594 & 0.541 & 1.000 & 0.744 & 0.496 & 0.487  \\
& 720 & \textbf{0.381} & \textbf{0.425} & 0.389 & 0.430 & 0.427 & 0.445 & 0.407 & 0.431 & 0.392 & 0.433 & 0.427 & 0.445 & 0.420 & 0.440 & 0.431 & 0.446 & 1.104 & 0.763 & 0.874 & 0.679 & 0.462 & 0.468 & 0.831 & 0.657 & 1.249 & 0.838 & 0.463 & 0.474 \\
\cmidrule{2-30}
& Avg & \textbf{0.313} & \textbf{0.372} & \underline{0.321} & \underline{0.377} & 0.377 & 0.403 & 0.361 & 0.391 & 0.322 & 0.379 & 0.383 & 0.407 & 0.374 & 0.398 & 0.387 & 0.407 & 0.942 & 0.684 & 0.611 & 0.550 & 0.414 & 0.427 & 0.559 & 0.515 & 0.954 & 0.723 & 0.437 & 0.449\\
\midrule
\multirow{5}{*}{\rotatebox{90}{ECL}} & 96 & \textbf{0.127} & \textbf{0.223} & 0.132 & 0.234 & 0.139 & 0.234 & 0.165 & 0.253 & 0.129 & 0.226 & 0.148 & 0.240 & 0.201 & 0.281 & 0.181 & 0.270 & 0.219 & 0.314 & 0.237 & 0.329 & 0.168 & 0.272 & 0.197 & 0.282 & 0.247 & 0.345 & 0.193 & 0.308\\
& 192 & 0.147 & 0.234 & 0.144 & \textbf{0.223} & 0.156 & 0.249 & 0.173 & 0.262 & \textbf{0.143} & 0.239 & 0.162 & 0.253 & 0.201 & 0.283 & 0.188 & 0.274 & 0.231 & 0.322 & 0.236 & 0.330 & 0.184 & 0.289 & 0.196 & 0.285 & 0.257 & 0.355 & 0.201 & 0.315 \\
& 336 & 0.165 & 0.266 & \textbf{0.156} & \textbf{0.259} & 0.172 & 0.267 & 0.188 & 0.277 & 0.161 & \textbf{0.259} & 0.178 & 0.269 & 0.215 & 0.298 & 0.204 & 0.293 & 0.246 & 0.337 & 0.249 & 0.344 & 0.198 & 0.300 & 0.209 & 0.301 & 0.269 & 0.369 & 0.214 & 0.329 \\
& 720 & 0.196 & 0.295 & \textbf{0.184} & \textbf{0.280} & 0.209 & 0.299 & 0.214 & 0.305 & 0.191 & 0.286 & 0.225 & 0.317 & 0.257 & 0.331 & 0.246 & 0.324 & 0.280 & 0.363 & 0.284 & 0.373 & 0.220 & 0.320 & 0.245 & 0.333 & 0.299 & 0.390 & 0.246 & 0.355 \\
\cmidrule{2-30}
& Avg & 0.158 & 0.254 & \textbf{0.154} & \textbf{0.249} & 0.169 & 0.262 & 0.185 & 0.274 & \underline{0.156} & \underline{0.253} & 0.178 & 0.270 & 0.219 & 0.298 & 0.205 & 0.290 & 0.244 & 0.334 & 0.251 & 0.344 & 0.192 & 0.295 & 0.212 & 0.300 & 0.268 & 0.365 & 0.214 & 0.327 \\
\midrule
\multirow{5}{*}{\rotatebox{90}{Exchange}} & 96 & 0.085 & 0.204 & \textbf{0.077} & 0.198 & 0.085 & 0.204 & - & - & 0.080 & \textbf{0.196} & 0.086 & 0.206 & 0.093 & 0.217 & 0.088 & 0.205 & 0.256 & 0.367 & 0.094 & 0.218 & 0.107 & 0.234 & 0.088 & 0.218 & 0.267 & 0.396 & 0.148 & 0.278 \\
& 192 & 0.179 & 0.304 & \textbf{0.159} & \textbf{0.270} & 0.178 & 0.299 & - & - & 0.166 & 0.288 & 0.177 & 0.299 & 0.184 & 0.307 & 0.176 & 0.299 & 0.470 & 0.509 & 0.184 & 0.307 & 0.226 & 0.344 & 0.176 & 0.315 & 0.351 & 0.459 & 0.271 & 0.315  \\
& 336 & 0.330 & 0.415 & 0.311 & \textbf{0.344} & 0.328 & 0.414 & - & - & 0.307 & 0.398 & 0.331 & 0.417 & 0.351 & 0.432 & \textbf{0.301} & 0.397 & 1.268 & 0.883 & 0.349 & 0.431 & 0.367 & 0.448 & 0.313 & 0.427 & 1.324 & 0.853 & 0.460 & 0.427 \\
& 720 & \textbf{0.653} & \textbf{0.580} & 0.697 & 0.623 & 0.817 & 0.679 & - & - & 0.656 & 0.582 & 0.847 & 0.691 & 0.886 & 0.714 & 0.901 & 0.714 & 1.767 & 1.068 & 0.852 & 0.698 & 0.964 & 0.746 & 0.839 & 0.695 & 1.058 & 0.797 & 1.195 & 0.695  \\
\cmidrule{2-30}
& Avg & \underline{0.311} & 0.375 & \underline{0.311} & \textbf{0.358} & 0.352 & 0.399 & - & - & \textbf{0.302} & \underline{0.366} & 0.360 & 0.403 & 0.378 & 0.417 & 0.367 & 0.404 & 0.940 & 0.707 & 0.370 & 0.413 & 0.416 & 0.443 & 0.354 & 0.414 & 0.750 & 0.626 & 0.519 & 0.429  \\
\midrule
\multirow{5}{*}{\rotatebox{90}{Traffic}} & 96 & 0.369 & 0.254 & \textbf{0.366} & \textbf{0.248} & - & - & 0.468 & 0.268 & 0.368 & 0.253 & 0.395 & 0.268 & 0.649 & 0.389 & 0.462 & 0.295 & 0.522 & 0.290 & 0.805 & 0.493 & 0.593 & 0.321 & 0.650 & 0.396 & 0.788 & 0.499 & 0.587 & 0.366 \\
& 192 & 0.381 & 0.278 & 0.394 & 0.292 & - & - & 0.413 & 0.317 & \textbf{0.379} & \textbf{0.261} & 0.417 & 0.276 & 0.601 & 0.366 & 0.466 & 0.296 & 0.530 & 0.293 & 0.756 & 0.474 & 0.617 & 0.336 & 0.598 & 0.370 & 0.789 & 0.505 & 0.604 & 0.373 \\
& 336 & \textbf{0.401} & 0.278 & 0.409 & 0.311 & - & - & 0.529 & 0.284 & 0.397 & \textbf{0.270} & 0.433 & 0.283 & 0.609 & 0.369 & 0.482 & 0.304 & 0.558 & 0.305 & 0.762 & 0.477 & 0.629 & 0.336 & 0.605 & 0.373 & 0.797 & 0.508 & 0.621 & 0.383 \\
& 720 & 0.444 & 0.295 & 0.443 & \textbf{0.294} & - & - & 0.564 & 0.297 & \textbf{0.440} & 0.296 & 0.467 & 0.302 & 0.647 & 0.387 & 0.514 & 0.322 & 0.589 & 0.328 & 0.719 & 0.449 & 0.640 & 0.350 & 0.645 & 0.394 & 0.841 & 0.523 & 0.626 & 0.382  \\
\cmidrule{2-30}
& Avg & \textbf{0.398} & \underline{0.276} & \underline{0.403} & 0.286 & - & - & 0.493 & 0.291 & \textbf{0.398} & \textbf{0.270} & 0.428 & 0.282 & 0.626 & 0.378 & 0.481 & 0.304 & 0.550 & 0.304 & 0.760 & 0.473 & 0.620 & 0.336 & 0.625 & 0.383 & 0.804 & 0.509 & 0.610 & 0.376 \\
\midrule
\multirow{5}{*}{\rotatebox{90}{Weather}} & 96 & 0.150 & 0.201 & \textbf{0.146} & 0.206 & 0.166 & 0.207 & 0.176 & 0.219 & 0.149 & \textbf{0.200} & 0.174 & 0.214 & 0.192 & 0.232 & 0.177 & 0.218 & 0.158 & 0.230 & 0.202 & 0.261 & 0.172 & 0.220 & 0.196 & 0.255 & 0.221 & 0.306 & 0.217 & 0.296 \\
& 192 & 0.198 & 0.245 & \textbf{0.189} & \textbf{0.239} & 0.216 & 0.254 & 0.222 & 0.260 & 0.196 & 0.245 & 0.221 & 0.254 & 0.240 & 0.271 & 0.225 & 0.259 & 0.206 & 0.277 & 0.242 & 0.298 & 0.219 & 0.261 & 0.237 & 0.296 & 0.261 & 0.340 & 0.276 & 0.336 \\
& 336 & 0.249 & 0.285 & 0.244 & 0.281 & 0.273 & 0.296 & 0.275 & 0.297 & \textbf{0.238} & \textbf{0.277} & 0.278 & 0.296 & 0.292 & 0.307 & 0.278 & 0.297 & 0.272 & 0.335 & 0.287 & 0.335 & 0.280 & 0.306 & 0.283 & 0.335 & 0.309 & 0.378 & 0.339 & 0.380 \\
& 720 & 0.312 & 0.326 & \textbf{0.297} & \textbf{0.309} & 0.351 & 0.346 & 0.350 & 0.349 & 0.314 & 0.334 & 0.358 & 0.347 & 0.364 & 0.353 & 0.354 & 0.348 & 0.398 & 0.418 & 0.351 & 0.386 & 0.365 & 0.359 & 0.345 & 0.381 & 0.377 & 0.427 & 0.403 & 0.428 \\
\cmidrule{2-30}
& Avg & 0.227 & \underline{0.264} & \textbf{0.219} & \textbf{0.258} & 0.251 & 0.276 & 0.255 & 0.280 & \underline{0.224} & \underline{0.264} & 0.258 & 0.278 & 0.272 & 0.291 & 0.259 & 0.281 & 0.259 & 0.315 & 0.271 & 0.320 & 0.259 & 0.287 & 0.265 & 0.317 & 0.292 & 0.363 & 0.309 & 0.360 \\
\bottomrule
\end{tabular}
}
\end{table*}

\begin{table*}[ht]
\centering
\caption{Full results of short-term forecasting in the M4 dataset.}
\label{appendix:short_term}
\renewcommand{\arraystretch}{1.2}
\resizebox{\textwidth}{!}{
\begin{tabular}{c|c|cccccccccccccccc}
\toprule
 \multicolumn{2}{c|}{} & Ours & Chimera & ModernTCN & PatchTST & TimesNet & N-HiTS & N-BEATS & ETSformer & LightTS & DLinear & FEDformer & Stationary & Autoformer & Pyraformer & Informer & Reformer \\
\midrule 
\multirow{3}{*}{\rotatebox{90}{Yearly}} & SMAPE & 13.150 & \textbf{13.107} & 13.226 & 13.258 & 13.387 & 13.418 & 13.436 & 18.009 & 14.247 & 16.965 & 13.728 & 13.717 & 13.974 & 15.530 & 14.727 & 16.169 \\
& MASE & 2.950 &  \textbf{2.902} &  2.957 &  2.985 &  2.996 &  3.045 &  3.043 &  4.487 &  3.109 &  4.283 &  3.048 &  3.078 &  3.134 &  3.711 &  3.418 &  3.800 \\
& OWA & 0.773 & \textbf{0.767} &  0.777 &  0.781 &  0.786 &  0.793 &  0.794 &  1.115 &  0.827 &  1.058 &  0.803 &  0.807 &  0.822 &  0.942 &  0.881 &  0.973 \\
\midrule
\multirow{3}{*}{\rotatebox{90}{Quarterly}} & SMAPE & 9.949 &  \textbf{9.892} &  9.971 & 10.179 & 10.100 & 10.202 & 10.124 & 13.376 & 11.364 & 12.145 & 10.792 & 10.958 & 11.338 & 15.449 & 11.360 & 13.313 \\
& MASE & 1.160 & 1.105 &  1.167 &  \textbf{0.803} &  1.182 &  1.194 &  1.169 &  1.906 &  1.328 &  1.520 &  1.283 &  1.325 &  1.365 & 2.350 &  1.401 &  1.775 \\
& OWA  & 0.875 & 0.853 &  0.878 &  \textbf{0.803} &  0.890 &  0.899 &  0.886 &  1.302 &  1.000 &  1.106 &  0.958 &  0.981 &  1.012 &  1.558 &  1.027 &  1.252 \\
\midrule
\multirow{3}{*}{\rotatebox{90}{Monthly}} & SMAPE & 12.563 & \textbf{12.549} & 12.556 & 12.641 & 12.670 & 12.791 & 12.677 & 14.588 & 14.014 & 13.514 & 14.260 & 13.917 & 13.958 & 17.642 & 14.062 & 20.128 \\
& MASE & 0.922 & \textbf{0.914} &  0.917 &  0.930 &  0.933 &  0.969 &  0.937 &  1.368 &  1.053 &  1.037 &  1.102 &  1.097 &  1.103 &  1.913 &  1.141 &  2.614 \\
& OWA  & 0.869 & \textbf{0.864} &  0.866 &  0.876 &  0.878 &  0.899 &  0.880 &  1.149 &  0.981 &  0.956 &  1.012 &  0.998 &  1.002 &  1.511 &  1.024 &  1.927 \\
\midrule
\multirow{3}{*}{\rotatebox{90}{Others}} & SMAPE & 4.708 &  \textbf{4.685} &  4.715 &  4.946 &  4.891 &  5.061 &  4.925 &  7.267 & 15.880 &  6.709 &  4.954 &  6.302 &  5.485 & 24.786 & 24.460 & 32.491 \\
& MASE & 3.165 &  3.007 &  3.107 &  \textbf{2.985} &  3.302 &  3.216 &  3.391 &  5.240 & 11.434 &  4.953 &  3.264 & 4.064 &  3.865 & 18.581 & 20.960 & 33.355 \\
& OWA  & 0.994 & 0.983 &  \textbf{0.986} &  1.044 &  1.035 &  1.040 &  1.053 &  1.591 &  3.474 &  1.487 &  1.036 &  1.304 &  1.187 &  5.538 &  5.013 &  8.679 \\
\midrule
\multirow{3}{*}{\rotatebox{90}{\tiny Weighted \scriptsize Avg}} & SMAPE & \underline{11.686} & \textbf{11.618} & 11.698 & 11.807 & 11.829 & 11.927 & 11.851 & 14.718 & 13.525 & 13.639 & 12.840 & 12.780 & 12.909 & 16.987 & 14.086 & 18.200 \\
& MASE & \underline{1.549} & \textbf{1.528} &  1.556 &  1.590 &  1.585 &  1.613 &  1.599 &  2.408 &  2.111 &  2.095 &  1.701 &  1.756 &  1.771 &  3.265 &  2.718 &  4.223 \\
& OWA & \underline{0.832} &  \textbf{0.827} &  0.838 &  0.851 &  0.851 &  0.861 &  0.855 &  1.172 &  1.051 &  1.051 &  0.918 &  0.930 &  0.939 &  1.480 &  1.230 &  1.775 \\
\bottomrule
\end{tabular}
}
\end{table*}

\begin{table*}[ht]
\centering
\caption{Full results of the classification task (accuracy \%)}
\label{appendix:classification}
\setlength{\tabcolsep}{3pt}
\renewcommand{\arraystretch}{1.2}
\resizebox{\textwidth}{!}{
\begin{tabular}{l|cccccccccccccccccc}
\toprule
\multicolumn{1}{l|}{Datasets} & LSTM & LSTNet & LSSL & Reformer & Informer & Pyraformer & Autoformer & Stationary & FEDformer & ETSformer & Flowformer & DLinear & LightTS & TimesNet & PatchTST & MTCN & Chimera & Ours \\ 
\midrule
EthanolConcentration & 32.3 & 39.9 & 31.1 & 31.9 & 31.6 & 30.8 & 31.6 & 32.7 & 31.2 & 28.1 & 33.8 & 32.6 & 29.7 & 35.7 & 32.8 & 36.3 & 39.8 & 35.7 \\
FaceDetection        & 57.7 & 65.7 & 66.7 & 68.6 & 67.0 & 65.7 & 68.4 & 68.0 & 66.0 & 66.3 & 67.6 & 68.0 & 67.5 & 68.6 & 68.3 & 70.8 & 70.4 & 69.6 \\
Handwriting          & 15.2 & 25.8 & 24.6 & 27.4 & 32.8 & 29.4 & 36.7 & 31.6 & 28.0 & 32.5 & 33.8 & 27.0 & 26.1 & 32.1 & 29.6 & 30.6 & 32.9 & 35.3 \\
Heartbeat            & 72.2 & 77.1 & 72.7 & 77.1 & 80.5 & 75.6 & 74.6 & 73.7 & 73.7 & 71.2 & 77.6 & 75.1 & 75.1 & 78.0 & 74.9 & 77.2 & 81.3 & 78.0 \\
JapaneseVowels       & 79.7 & 98.1 & 98.4 & 97.8 & 98.9 & 98.4 & 96.2 & 99.2 & 98.4 & 95.9 & 98.9 & 96.2 & 96.2 & 98.4 & 97.5 & 98.8 & 99.1 & 98.4 \\
PEMS-SF              & 39.9 & 86.7 & 86.1 & 82.7 & 81.5 & 83.2 & 82.7 & 87.3 & 80.9 & 86.0 & 83.8 & 75.1 & 88.4 & 89.6 & 89.3 & 89.1 & 89.5 & 88.1 \\
SelfRegulationSCP1   & 68.9 & 84.0 & 90.8 & 90.4 & 90.1 & 88.1 & 84.0 & 89.4 & 88.7 & 89.6 & 92.5 & 87.3 & 89.8 & 91.8 & 90.7 & 93.4 & 93.7 & 92.2\\
SelfRegulationSCP2   & 46.6 & 52.8 & 52.2 & 56.7 & 53.3 & 53.3 & 50.6 & 57.2 & 54.4 & 55.0 & 56.1 & 50.5 & 51.1 & 57.2 & 57.8 & 60.3 & 59.9 & 57.8 \\
SpokenArabicDigits   & 31.9 & 100.0 & 100.0 & 97.0 & 100.0 & 99.6 & 100.0 & 100.0 & 100.0 & 100.0 & 98.8 & 81.4 & 100.0 & 99.0 & 98.3 & 98.7 & 100.0 & 99.0 \\
UWaveGestureLibrary  & 41.2 & 87.8 & 85.9 & 85.6 & 85.6 & 83.4 & 85.9 & 87.5 & 85.3 & 85.0 & 86.6 & 82.1 & 80.3 & 85.3 & 85.8 & 86.7 & 86.7 &  90.0 \\
\midrule
Average Accuracy     & 48.6 & 71.8 & 70.9 & 71.5 & 72.1 & 70.8 & 71.1 & 72.7 & 70.7 & 71.0 & 73.0 & 67.5 & 70.4 & 73.6 & 72.5 & 74.2 & 75.3 & 74.4\\
\bottomrule
\end{tabular}
}
\end{table*}

\begin{table*}[ht]
\centering
\caption{Full results of the anomaly detection task. Metrics include Precision (P), Recall (R), and F1-score.}
\label{appendix:anomaly_detection}
\setlength{\tabcolsep}{3pt}
\renewcommand{\arraystretch}{1.2}
\resizebox{\textwidth}{!}{
\begin{tabular}{l|ccc|ccc|ccc|ccc|ccc|c}
\toprule
 & \multicolumn{3}{c|}{SMD} & \multicolumn{3}{c|}{MSL} & \multicolumn{3}{c|}{SMAP} & \multicolumn{3}{c|}{SWaT} & \multicolumn{3}{c|}{PSM} & \multirow{2}{*}{Avg F1} \\ \cmidrule{2-16}
& P & R & F1 & P & R & F1 & P & R & F1 & P & R & F1 & P & R & F1 & \\
\midrule
LSTM         & 78.52 & 65.47 & 71.41 & 78.04 & 86.22 & 81.93 & 91.06 & 57.49 & 70.48 & 78.06 & 91.72 & 84.34 & 69.24 & 99.53 & 81.67 & 77.97 \\
Transformer  & 83.58 & 76.13 & 79.56 & 71.57 & 87.37 & 78.68 & 89.37 & 57.12 & 69.70 & 68.84 & 96.53 & 80.37 & 62.75 & 96.56 & 76.07 & 76.88 \\
LogTrans     & 83.46 & 70.13 & 76.21 & 73.05 & 87.37 & 79.57 & 89.15 & 57.59 & 69.97 & 68.67 & 97.32 & 80.52 & 63.06 & 98.00 & 76.74 & 76.60 \\
TCN          & 84.06 & 79.07 & 81.49 & 75.11 & 82.44 & 78.60 & 86.90 & 59.23 & 70.45 & 76.59 & 95.71 & 85.09 & 54.59 & 99.77 & 70.57 & 77.24 \\
Reformer     & 82.58 & 69.24 & 75.32 & 85.51 & 83.31 & 84.40 & 90.91 & 57.44 & 70.40 & 72.50 & 96.53 & 82.80 & 59.93 & 95.38 & 73.61 & 77.31 \\
Informer     & 86.60 & 77.23 & 81.65 & 81.77 & 86.48 & 84.06 & 90.11 & 57.13 & 69.92 & 70.29 & 96.75 & 81.43 & 64.27 & 96.33 & 77.10 & 78.83 \\
Anomaly*     & 88.91 & 82.23 & 85.49 & 79.61 & 87.37 & 83.31 & 91.85 & 58.11 & 71.18 & 72.51 & 97.32 & 83.10 & 68.35 & 94.72 & 79.40 & 80.50 \\
Pyraformer   & 85.61 & 80.61 & 83.04 & 83.81 & 85.93 & 84.86 & 92.54 & 57.71 & 71.09 & 87.92 & 96.00 & 91.78 & 71.67 & 96.02 & 82.08 & 82.57 \\
Autoformer   & 88.06 & 82.35 & 85.11 & 77.27 & 80.92 & 79.05 & 90.40 & 58.62 & 71.12 & 89.85 & 95.81 & 92.74 & 99.08 & 88.15 & 93.29 & 84.26 \\
LSSL         & 78.51 & 65.32 & 71.31 & 77.55 & 88.18 & 82.53 & 89.43 & 53.43 & 66.90 & 79.05 & 93.72 & 85.76 & 66.02 & 92.93 & 77.20 & 76.74 \\
Stationary   & 88.33 & 81.21 & 84.62 & 68.55 & 89.14 & 77.50 & 89.37 & 59.02 & 71.09 & 68.03 & 96.75 & 79.88 & 97.82 & 96.76 & 97.29 & 82.08 \\
DLinear      & 83.62 & 71.52 & 77.10 & 84.34 & 85.42 & 84.88 & 92.32 & 55.41 & 69.26 & 80.91 & 95.30 & 87.52 & 98.28 & 89.26 & 93.55 & 82.46 \\
ETSformer    & 87.44 & 79.23 & 83.13 & 85.13 & 84.93 & 85.03 & 92.25 & 55.75 & 69.50 & 90.02 & 80.36 & 84.91 & 99.31 & 85.28 & 91.76 & 82.87 \\
LightTS      & 87.10 & 78.42 & 82.53 & 82.40 & 75.78 & 78.95 & 92.58 & 55.27 & 69.21 & 91.98 & 94.72 & 93.33 & 98.37 & 95.97 & 97.15 & 84.23 \\
FEDformer    & 87.95 & 82.39 & 85.08 & 77.14 & 80.07 & 78.57 & 90.47 & 58.10 & 70.76 & 90.17 & 96.42 & 93.19 & 97.31 & 97.16 & 97.23 & 84.97 \\
TimesNet (I) & 87.76 & 82.63 & 85.12 & 82.97 & 85.42 & 84.18 & 91.50 & 57.80 & 70.85 & 88.31 & 96.24 & 92.10 & 98.22 & 92.21 & 95.21 & 85.49 \\
TimesNet (R) & 88.66 & 83.14 & 85.81 & 83.92 & 86.42 & 85.15 & 92.52 & 58.29 & 71.52 & 86.76 & 97.32 & 91.74 & 98.19 & 96.76 & 97.47 & 86.34 \\
CrossFormer  & 83.60 & 76.61 & 79.70 & 84.68 & 83.71 & 84.19 & 92.04 & 55.37 & 69.14 & 88.49 & 93.48 & 90.92 & 97.16 & 89.73 & 93.30 & 83.45 \\
PatchTST     & 87.42 & 81.65 & 84.44 & 84.07 & 86.23 & 85.14 & 92.43 & 57.51 & 70.91 & 80.70 & 94.93 & 87.24 & 98.87 & 93.99 & 96.37 & 84.82 \\
ModernTCN   & 87.86 & 83.85 & 85.81 & 83.94 & 85.93 & 84.92 & 93.17 & 57.69 & 71.26 & 91.83 & 95.98 & 93.86 & 98.09 & 96.38 & 97.23 & 86.62 \\
Chimera  & 87.74 & 83.29 & 85.46 & 84.01 & 86.83 & 85.39 & 93.05 & 58.12 & 71.55 & 92.18 & 95.93 & 94.01 & 97.30 & 96.19 & 96.74 & 86.69 \\
Ours  & 86.31 & 81.33 & 83.74 & 85.66 & 84.31 & 84.97 & 92.24 & 64.86 & 76.17 & 91.76 & 95.29 & 93.94 & 98.20 & 96.67 & 97.42 & 87.24\\
\bottomrule
\end{tabular}
}
\end{table*}

\end{document}